\documentclass[nohyperref]{article}

\usepackage{microtype}
\usepackage{graphicx}
\usepackage{subfigure}
\usepackage{booktabs} % for professional tables

\usepackage{hyperref}

\usepackage[utf8]{inputenc} % allow utf-8 input
\usepackage[T1]{fontenc}    % use 8-bit T1 fonts
\usepackage{hyperref}       % hyperlinks
\usepackage{url}            % simple URL typesetting
\usepackage{booktabs}       % professional-quality tables
\usepackage{amsfonts}       % blackboard math symbols
\usepackage{nicefrac}       % compact symbols for 1/2, etc.
\usepackage{microtype}      % microtypography
\usepackage{xcolor}         % colors
\usepackage{amsthm}
\usepackage{dsfont}
\usepackage{wrapfig}
\usepackage{amsmath}
\usepackage{graphicx}
\usepackage{algorithmic}
\usepackage{multirow}
\usepackage[ruled,vlined]{algorithm2e}
\usepackage{float}
\usepackage{environ, wrapfig}
\usepackage{multicol}
\usepackage{lipsum}

\NewEnviron{figuretest}[1]{% 
\stepcounter{figure}
\begin{wrapfigure}{l}{0pt}
#1
\end{wrapfigure}
\noindent{}Figure \thefigure:
\BODY}

\newcommand{\state}{{s}}
\newcommand{\action}{{a}}
\newcommand{\st}{{s}_t}
\newcommand{\at}{{a}_t}
\newcommand{\seq}{{s}_e}
\newcommand{\aeq}{{a}_e}
\newcommand{\stp}{{s}_{t+1}}
\newcommand{\atp}{{a}_{t+1}}
\newcommand{\argmax}{\text{argmax}}
\newcommand{\argmin}{\text{argmin}}

\newcommand{\clf}{W}
\newcommand{\ldm}{G}
\newcommand{\ldmlevelset}{\mathcal{G}_c}
\definecolor{byzantine}{rgb}{0.74, 0.2, 0.64}

\usepackage[toc,page]{appendix}

\usepackage[accepted]{icml2022}

\usepackage{amsmath}
\usepackage{amssymb}
\usepackage{mathtools}
\usepackage{amsthm}

\usepackage[capitalize,noabbrev]{cleveref}

\theoremstyle{plain}
\newtheorem{theorem}{Theorem}[section]
\newtheorem{proposition}[theorem]{Proposition}
\newtheorem{lemma}[theorem]{Lemma}
\newtheorem{corollary}[theorem]{Corollary}
\newtheorem{remark}[theorem]{Remark}
\theoremstyle{definition}
\newtheorem{definition}[theorem]{Definition}
\newtheorem{assumption}[theorem]{Assumption}
\newtheorem{example}[theorem]{Example}

\definecolor{DarkRed}{rgb}{0.7,0.1,0.1}
\definecolor{DarkBlue}{rgb}{0.1,0.1,0.7}

\usepackage[textsize=tiny]{todonotes}

\icmltitlerunning{Lyapunov Density Models}

\begin{document}

\twocolumn[
\icmltitle{Lyapunov Density Models: \\Constraining Distribution Shift in Learning-Based Control}




\begin{icmlauthorlist}
\icmlauthor{Katie Kang}{yyy}
\icmlauthor{Paula Gradu}{yyy}
\icmlauthor{Jason Choi}{yyy}
\icmlauthor{Michael Janner}{yyy}
\icmlauthor{Claire Tomlin}{yyy}
\icmlauthor{Sergey Levine}{yyy}
\end{icmlauthorlist}

\icmlaffiliation{yyy}{University of California, Berkeley}
\icmlcorrespondingauthor{Katie Kang}{katiekang@eecs.berkeley.edu }

\icmlkeywords{Machine Learning, ICML}

\vskip 0.3in
]



\printAffiliationsAndNotice{}  % leave blank if no need to mention equal contribution
%\printAffiliationsAndNotice{}%{\icmlEqualContribution} % otherwise use the standard text.

\begin{abstract}
Learned models and policies can generalize effectively when evaluated within the distribution of the training data, but can produce unpredictable and erroneous outputs on out-of-distribution inputs. In order to avoid distribution shift when deploying learning-based control algorithms, we seek a mechanism to constrain the agent to states and actions that resemble those that it was trained on. In control theory, Lyapunov stability and control-invariant sets allow us to make guarantees about controllers that stabilize the system around specific states, while in machine learning, density models allow us to estimate the training data distribution. Can we combine these two concepts, producing learning-based control algorithms that constrain the system to in-distribution states using only in-distribution actions? In this work, we propose to do this by combining concepts from Lyapunov stability and density estimation, introducing Lyapunov density models: a generalization of control Lyapunov functions and density models that provides guarantees on an agent's ability to stay in-distribution over its entire trajectory. 
\end{abstract}

% \vspace{-2em}

\section{Introduction}
Learning-based control algorithms, including model-free and model-based reinforcement learning, have shown the capability to surpass the performance of traditional model-based control on a number of complex, difficult-to-model systems \citep{kalashnikov2018qtopt,akkaya2019solving,li2021rlbipedal}, by leveraging data in place of approximate mathematical models. However, the data-dependent nature of these learning-enabled controllers also leads to new challenges. One such challenge is that learned models can only provide reliable predictions when queried within the training data distribution. If the inputs to a learning-based controller deviate far from its training data, the controller can suffer from \emph{model exploitation}, in which mispredictions from the machine learning model causes the controller to output suboptimal or even catastrophic action commands. 
%Thus, when deploying a learning-based controller, if the inputs that the controller receives at test-time deviate far from its training distribution, the output of the controller will be unpredictable, leading to potentially suboptimal or even catastrophic behavior. 

To mitigate the negative effects of distribution shift, can we design controllers that perform the desired task \emph{while ensuring the system remains in-distribution}?
%can ensure that a system remains in-distribution throughout its trajectory? 
% One approach is to use a density model that characterizes the likelihood of each point under the training distribution, constraining or regularizering the controller to prevent the agent from taking a low-likelihood action or visiting a low-likelihood state
One approach is to learn a density model of the training data, and use it to constrain or regularize the controller to prevent the agent from taking low-likelihood actions or visiting low-likelihood states
\cite{richter2017safe,mcallister2019robustness,fujimoto2019offpolicy,kumar2019stabilizing,wu2019behavior}. However, because density models are not aware of system dynamics, this kind of approach ``greedily'' chooses actions that are in-distribution at the next timestep, 
%however, density models are not aware of system dynamics, 
but does not provide a mechanism for \emph{staying} in-distribution over a long horizon. On the other hand, guaranteed long-horizon constraint satisfaction has been widely studied in control theory. Lyapunov functions map each state to a non-negative scalar which must decrease as the system evolves, providing a mechanism for ensuring the stability of a system over an infinite horizon. However, these notions of stability are formulated relative to fixed points of the system, and unrelated to any data or distribution.

The goal of our work is to provide a mechanism for ensuring that a system controlled by a learning-based policy will remain in-distribution throughout its trajectory. To this end, the main contribution of our work is the Lyapunov density model (LDM), a function that combines the data-aware aspect of density models with the dynamics-aware aspect of Lyapunov functions. Analogous to how control Lyapunov functions are used to guarantee that a controller stabilizes a system, we show how LDMs can be used to design controllers that are guaranteed to keep a system in-distribution, where na\"{i}vely using a density model will fail. We also present an algorithm for learning the LDM and its associated policy function directly from a training set of transitions, without assuming any knowledge of the underlying dynamics or active system interaction. Since learning the LDM is itself a data-dependent process which may suffer from distribution shift, we also provide theoretical analysis of our algorithm, showing that performing control with a LDM outperforms using a density model, even under approximation error. LDMs are a general tool for synthesizing and verifying controllers to stay in-distribution, and can be used in a variety of different learning-based control applications. We present a method of using an LDM in model-based RL as a constraint on the model optimizer, and evaluate this method empirically. The project webpage can be found at: \url{https://sites.google.com/berkeley.edu/ldm/}

%%SL.1.17: Don't forget to fill in the citations...

\section{Related works}
A large body of work in control theory focuses on providing guarantees about the evolution of dynamical systems. To guarantee stability, one could use Lyapunov functions for uncontrolled systems \citep{SastryShankar1999Ns}, and control Lyapunov functions for controlled systems \citep{SONTAG1989117}. To guarantee that an agent remains within a set of safe states, one could use control-barrier functions \citep{ames2017cbf} or Hamilton-Jacobi reachability value functions \citep{Bansal2017}. All of these approaches, in addition to our work, provide conditions for infinite-horizon guarantees based on the control invariance of the level sets of the functions. However, these methods usually assume access to the ground-truth dynamics model, whereas our method only uses data from the system. More importantly, these methods are primarily concerned with satisfying manually-designed constraints based on limitations of the physical system. In contrast, our method derives constraints from the training data distribution in order to avoid visiting out-of-distribution regions where the learning-based controller is prone to be erroneous. 

Some recent works take a data-driven approach to maintaining safety or stability%(including learning a Lyapunov function), 
, without assuming access to a ground-truth dynamics model \cite{berkenkamp2017safe, cheng2019end, dai2021lyapunov, chang2020neural, mehrjou2020learning, mittal2020neural, richards2018lyapunov, fisac2019bridging, robey2020learningcbf, tu2022sample}. Although this line of work shares superficial similarities with our work, such as learning a ``Lyapunov-like'' function % ``barrier-like'' value function
, the two serve orthogonal purposes. %Similar to the methods in the previous paragraph, t
The primary goal of such safe learning methods is to satisfy physical safety constraints, while our goal is to maintain the reliability of learning-based controllers by staying in-distribution. %In fact, one could potentially combine the two lines of work to produce a controller that can both ensure the physical safety of the system and reliable outputs from the learning-based controller. 

A number of prior methods, particularly in model-free and model-based offline RL~\citep{nair2021awac, laroche2019safe, kumar2019stabilizing,fujimoto2019offpolicy,kumar2020conservative,kidambi2020morel,yu2020mopo}, focus on learning a good policy from offline data by avoiding distributional shift. Although we also aim to learn from offline data and avoid distributional shift, our focus is different. Our aim is to learn a general function that can prevent distributional shift for any downstream controller. The Lyapunov density model is not associated with any task, its role is only to provide a constraint such that any downstream policy that satisfies this constraint avoids excessive distributional shift. %This makes it possible to combine our approach with a range of downstream control methods, including imitation learning and model-predictive control. % Furthermore, we provide a theoretical analysis that compares the performance of our proposed dynamics-aware Lyapunov density model with a more standard density-based constraint, which is roughly analogous to approaches used in prior offline RL methods \textcolor{red}{[citations]}.

\section{Problem Formulation} \label{sec:problem_formulation}
Consider a deterministic, continuous-state, discrete-time dynamical system $\stp = f(\st, \at)$ with state space $\mathcal{S}$ and action space $\mathcal{A}$. We start with a dataset of transitions $\{(\st^i, \at^i, \stp^i)\}_{i=1}^N$ generated from a distribution $P(\mathrm{s}, \mathrm{a})$, which is used to train a learning-enabled control system. We will say a state-action tuple $(\st,\at)$ is \emph{in-distribution} if $P(\st, \at) \geq c$, for some density level $c > 0$. We denote the set of in-distribution states and actions as $\mathcal{D}_c := \{(\st, \at) | P(\st, \at) \geq c\}$. We wish to query our learned model or policy, used for performing a downstream task,
only on points in $\mathcal{D}_c$, to mitigate distribution shift at test-time. While being in-distribution or in-support does not \emph{necessarily} guarantee that the learned models or policies will be accurate, this condition is simple and has been used widely in prior work \citep{fujimoto2019offpolicy, laroche2019safe, kumar2019stabilizing, dean2020guaranteeing}. 
% Though there may be other reasonable notions of in-distribution or in-support, given the prevalence of this viewpoint in prior work, we take it as axiomatic that we can expect the reliability of a learning-based model at a state-action pair to be proportional to its density (implying our model has less error in a region the more it is explored/observed in prior data). Therefore, our main goal will be to restrict the system to \emph{always} stay in the set $\mathcal{D}_c$.

% In density modeling, the goal is to approximate a distribution $P(x)$ using a model $P_\theta(x)$, trained with a dataset $\{x_i\}_{i=1}^N$ generated from $P(x)$, where $x$ is some random variable -- in our case, tuples $(\st,\at)$. This problem is often formulated as maximum likelihood estimation, where the aim is to find the best model $P_{\theta^*}$ within a function class such that $\theta^* = \argmax_{\theta \in \Theta} \frac{1}{N}\sum_{i=1}^N \log P_{\theta}(x_i)$. Modern density estimation techniques include energy-based models \citep{hinton2002training} and flow models \textcolor{red}{[citations]}, and other approaches \textcolor{red}{[citations]}.

Because we do not have access to the ground truth training data distribution $P(\mathrm{s}, \mathrm{a})$, we could instead approximate it with a density model $P_\theta(\st, \at)$, using techniques such as energy-based models \citep{hinton2002training} and flow models \cite{durkan2019neural}. However, even if $P_\theta(\st, \at)$ is learned effectively, it is not obvious how to design a controller such that $P_\theta(\st, \at) \geq c$ is satisfied for an agent's entire trajectory. The main challenge lies in the sequential nature of control problems: a controller that satisfies $P_\theta(\st, \at) \geq c$ at the time $t$ could still take the agent to a future state $s_{t+n}$ for which there exists no action $a_{t+n}$ that satisfies $P_\theta(s_{t+n}, a_{t+n}) \geq c$, inevitably driving the agent into a low density region. See Figure~\ref{fig_density_example} for an illustrative example. 

\begin{figure}
\begin{center}
\includegraphics[width=0.8\columnwidth, trim={0 0 0 0},clip]{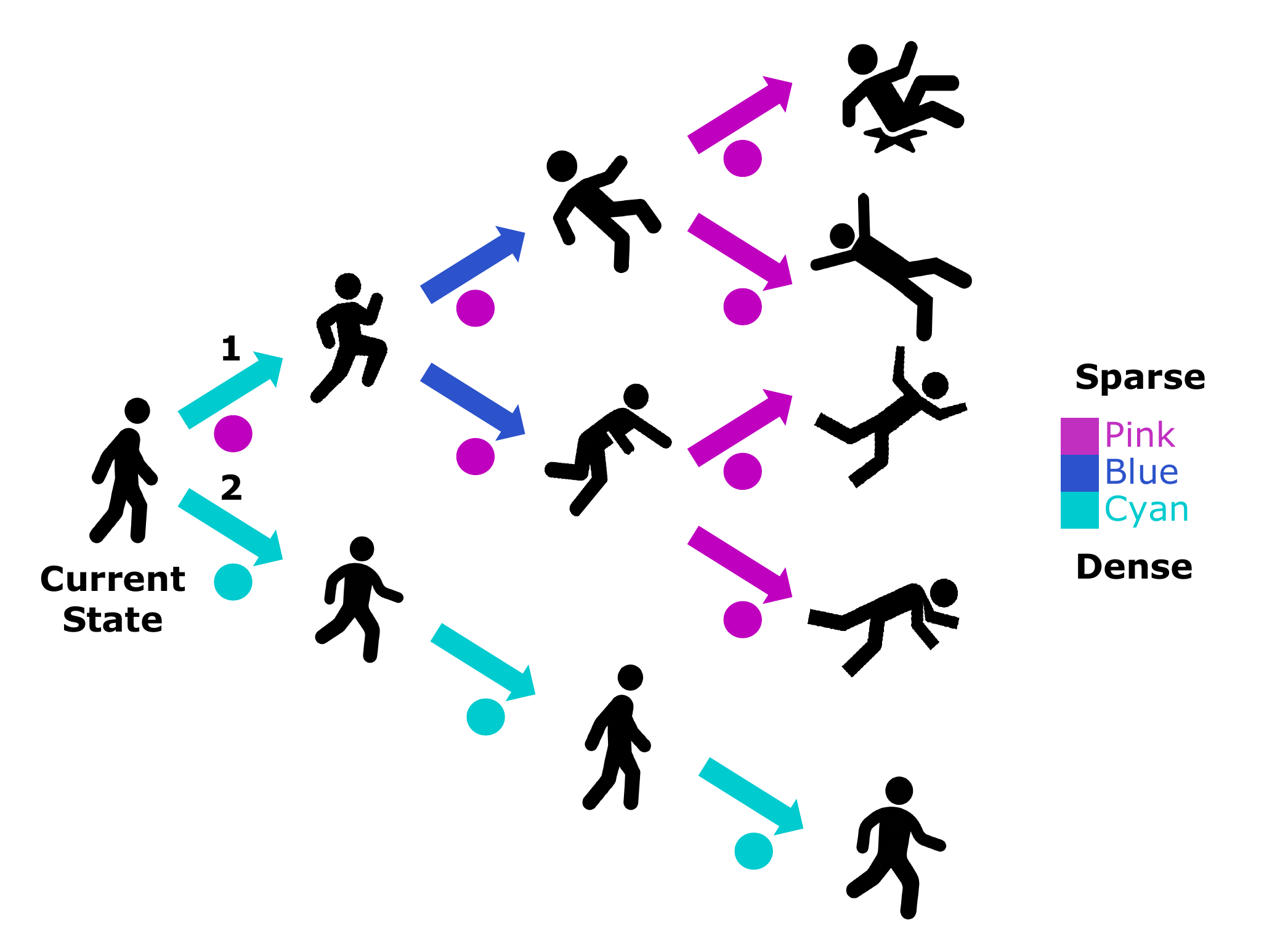}
\end{center}
\vspace{-1.5em}
\caption{Illustrative example of how naively constraining the agent with a density model can lead to failure. Arrows denote possible transitions, and color denotes the likelihood of a state and action under the data distribution. % and the coloring of the arrows denotes the data distribution (darker being higher density). % Here, the human figures denotes states of the system, the arrows denote possible transitions, and the coloring of the arrows denotes the data distribution. 
Suppose our goal was for the agent to remain at a higher density than blue. Using the density model as a constraint on the actions ($P_\theta(\st, \at) \geq c_{blue}$) would allow the agent to take action 1 or 2 from the current state. However, taking action 1 will inevitably cause the agent to go to the pink states, leading to failure. Instead, constraining the agent's actions with the lowest density that the agent will encounter in the \emph{future} (denoted by coloring of dots below arrows) will restrict the agent to only take action 2 at the current state, 
%Using the lowest density that the agent will encounter in the \emph{future} (denoted by coloring of dots below arrows) to constrain the actions will only allow the agent to move left, 
which ensures the agent stays above the blue density threshold throughout its trajectory. In order to \emph{remain} in-distribution, the agent must not only avoid states and actions that are currently unfamiliar, but also those that may lead to unfamiliar states and actions \emph{in the future}.}
% In comparison, using a "long-horizon" version of the density model (denoted by coloring of dots below arrow) as a constraint would successfully keep the agent in distribution.}
\vspace{-1em}
\label{fig_density_example}
\end{figure}

\section{Background on CLFs}
\label{CLF_prelim_subsection}
In order to perform long-horizon reasoning about the data distribution in conjunction with the system dynamics, we turn to a widely studied tool in controls: control Lyapunov functions (CLFs). CLFs are a tool for verifying that a system is stabilizable, meaning that system can be controlled to remain close to an equilibrium point over an infinite horizon. First, we introduce formal definitions for an equilibrium point, stabilizability, and a CLF.% In this work, we only only consider the discrete-time setting, which is more commonly studied in reinforcement learning, though the analysis could also be extended to continuous-time systems. 
% \footnote{While many different notions of stabilizability exist~\citep{Camilli2008,ames2014rapidly}, we focus on the weaker notion of stabilizability given by Def.~\ref{def:sisl} below, since this property is sufficient for our problem. }%Furthermore, CLFs are usually defined for asymptotic stabilizability \cite{SONTAG1989117}, but we generalize the definition to reason about stabilizablity in the sense of Lyapunov.}

\begin{definition}[Equilibrium Point]A state $s_e$ is an equilibrium point if $\exists$ action $a_e\in\mathcal{A}$ such that $f(s_e, a_e)=s_e$.
\end{definition}

\begin{definition}[Stabilizable in the sense of Lyapunov]
\label{def:sisl}
A system is stabilizable (in the sense of Lyapunov) if \;$\forall \epsilon > 0$, $\exists \delta$ such that for all $s_0\in\mathcal{S}$ such that $||s_0 - s_e||\leq \delta$, there exists $\{a_t\}_{t=0}^\infty$ such that the resulting $\{s_t\}_{t=0}^\infty$ satisfies $||s_t - s_e||\leq \epsilon$ $\forall t \geq 0$. \citep[Ch.4.4]{murray2017mathematical}
%%SL.1.17: this citation and other citations in this section seem to be malformed in the pdf
\end{definition}

\begin{definition}[Control Lyapunov Function]

A continuous and radially unbounded function $\clf$: $\mathcal{S} \rightarrow \mathbb{R}$ is a control Lyapunov function if the following conditions hold:
\vspace{-1em}
\begin{enumerate}
\item $\forall {s} \in \mathcal{S}$, $\exists {a} \in \mathcal{A}$ s.t. $\clf({s}) \geq \clf(f({s}, {a}))$,
\vspace{-0.5em}
\item $\forall {s} \neq s_e$, $\clf({s}) > 0$,
\vspace{-0.5em}
\item $\clf(s_e) = 0$.
\end{enumerate}
\label{def:clf}
\end{definition}
\vspace{-1em}

A CLF maps each state to a scalar ``energy level". The first condition of CLFs mandates that the system can always be controlled to remain at or below its current energy level.
This means that for an energy threshold, if an agent is inside the set of states with energies below that threshold, then it can be controlled to remain within that set for all time. This property of CLFs is formalized below. 

\begin{definition}[Control-Invariant Set]
A set of states $\mathcal{C} \subseteq \mathcal{S}$ is called a control-invariant set if\; $\forall {s}_0 \in \mathcal{C}$, $\exists \{\at\}_{t=0}^\infty$ such that\;$\st \in \mathcal{C}$ $\forall t\geq 0$. 
\citep[Ch.10.9]{borrelli2017predictive}
\end{definition}

\begin{theorem}
For a control Lyapunov function $\clf$, the sub-level set $\{{s}\;|\;\clf({s})\leq c\}$ is control-invariant for all $c \geq 0$. 
\label{level-sets-thm}
\end{theorem}
\vspace{-1em}
Combining Theorem \ref{level-sets-thm} with the second and third conditions of CLFs, which mandates that the equilibrium point has the lowest energy, we can use the sublevel sets of the CLF to extract arbitrarily small control-invariant sets which contain the equilibrium point. This ensures that the agent is never able to deviate far from the equilibrium point, thereby guaranteeing stabilizability. 

\begin{theorem}
A system has a CLF if there exists an action sequence $\{a_t\}_{t=0}^\infty$ such that the resulting trajectory $\{s_t\}_{t=0}^\infty$ is stabilizable in the sense of Lyapunov with respect to $x_e$. \citep[Theorem~4.4.4]{murray2017mathematical}
\label{lyapunov_theorem}
\end{theorem}
% CT 5.27 is this theorem only in discrete time?  You could imagine it not working in continuous time.

\vspace{-0.1in}
\section{Lyapunov Density Models}
\label{sec4_ldms}
Density models provide us with an estimate of the training data distribution, while CLFs allow us to make guarantees about the long-horizon evolution of a dynamical system. In this section, we derive Lyapunov density models, a union of the two ideas which can be used to guarantee a system's ability to stay within regions of high data density.

In order for a learned model or policy to be queried only on points in $\mathcal{D}_c = \{(\st, \at) | P(\st, \at) \geq c\}$ when evaluated in closed-loop, the agent must only visit states from which it can be controlled to remain in $\mathcal{D}_c$.
%there must exist a sequence of actions $ \{\at\}_{t=1}^\infty$ that induces a trajectory $\{\st\}_{t=1}^\infty$
% $(\{\st\}_{t=1}^\infty, \{\at\}_{t=1}^\infty)$ such that $(\st, \at) \in \mathcal{D}_c$ for all $t \geq 0$. 
This desired property is exhibited by a control-invariant set with a state-dependent action space of $\mathcal{A}(\st) = \{\at |P(\st, \at) \geq c\}$. We define the notion of a control-invariant set in the joint state and action space to simplify notation. 

\begin{definition}[State-Action Control-Invariance]
A set $\mathcal{G} \subseteq \mathcal{S} \times \mathcal{A}$ is a state-action control-invariant set if \;$\forall ({s}_0, {a}_0) \in \mathcal{G}$, $\exists \{\at\}_{t=1}^\infty$
such that $(\st, \at) \in \mathcal{G}$ $\forall t\geq 0$.
\label{gci}
\end{definition}
%%SL.1.17: could we simply move this definition to the preceding section right where we define control invariant set, just to keep these all in one place?

Let $\ldmlevelset$ be a state-action control-invariant set inside $\mathcal{D}_c$. By Definition \ref{gci}, if an agent is in state $s_0$ and takes action $a_0$ such that $(s_0, a_0) \in \ldmlevelset$, then there exists a trajectory $(\{\st\}_{t=1}^\infty, \{\at\}_{t=1}^\infty)$ that can satisfies $P(\st, \at) \geq c$ for all $t=0, .., \infty$, which guarantees that the agent is able to remain in-distribution \emph{for all future time steps}.

How can we construct a state-action control-invariant set $\ldmlevelset$? Recall we can obtain a control-invariant set associated with \emph{each} level of a CLF (Theorem~\ref{level-sets-thm}). We adapt this property to derive a function such that 1) each sublevel-set is a state-action control-invariant set, and 2) each level of the function is associated with a density level of a probability distribution. We call this a Lyapunov density model (LDM).

% By constructing a state-action control-invariant set $\ldmlevelset$, we can verify that an agent's trajectory can stay above a single density threshold $c$. However, $c$ can take on a continuum of values, depending on the desired level of conservatism. How can we verify an agent's ability to stay above \emph{any} density level above 0? To this end, we take inspiration from the Lyapunov theory. Recall that given a CLF, we can obtain a control-invariant set associated with \emph{each} level of the function (Theorem~\ref{level-sets-thm}). We adapt this property to derive a function defined over the state and action space such that 1) each sublevel-set is a state-action control-invariant set, and 2) each level of the function is associated with a density level of a probability distribution. We call this a Lyapunov density model (LDM).

\begin{definition}[Lyapunov Density Models]\label{def:ldm}
A function $\ldm(\st, \at)$: $\mathcal{S} \times \mathcal{A} \rightarrow \mathbb{R}$ continuous in $\st$ and $\at$ is a Lyapunov Density Model for the dynamical system $\stp = f(\st, \at)$ and a density function $P:\mathcal{S}\times\mathcal{A}\rightarrow\mathbb{R}^{+}$ if the following conditions hold:
\vspace{-1em}
\begin{enumerate}
\item $\forall (\st, \at) \in \mathcal{S} \times \mathcal{A}$, $\exists \atp \in \mathcal{A}$ \;such that\;$\ldm(\st, \at) \geq \ldm(f(\st, \at), \atp)$ 
\vspace{-0.5em}
\item $\forall (\st, \at) \in \mathcal{S} \times \mathcal{A}$, $\ldm(\st, \at) \geq -\log(P(\st, \at))$ 
\end{enumerate}
\label{ldm-definition}
\end{definition}
\vspace{-1em}

We formalize the key property of Lyapunov Density Models (LDMs) in Theoremm~\ref{thm:ldm_base_guarantee} below, proven in Appendix~\ref{appendix:sec4_proofs}.
\begin{theorem}\label{thm:ldm_base_guarantee}
Any sub-level set of an LDM $\ldm$, \mbox{$\ldmlevelset := \{(\st, \at): \ldm(\st, \at) \leq -\log(c)\}$}, is a state-action control-invariant set s.t. $\forall (\st, \at) \in \ldmlevelset$, $P(\st, \at) \geq c$.
% with $P(\st, \at) \geq p$ $\forall (\st, \at) \in \mathcal{Q}_p$.
\label{main-theorem}
\end{theorem}

An LDM $\ldm(\st, \at)$ is determined by a dynamical system $\stp = f(\st, \at)$ and a data distribution $P(\st, \at)$ \footnote{LDMs are, in fact, a generalization of both Lyapunov functions and density models; for more details, see Appendix \ref{appdx_connections_to_density_models}.}, and maps each state and action pair to a scalar value representing a negative log density level $-\log(c)$. By Theorem~\ref{main-theorem}, we can simultaneously represent a state-action control-invariant set inside $\mathcal{D}_c$ for \emph{all} values of $c$ between $0$ and $\max_{\st, \at} P(\st, \at)$ using the sub-level sets of $\ldm(\st, \at)$, which we denote as $\ldmlevelset$. Since $\ldmlevelset$ is a state-action control-invariant set, as long as the agent only visits states and executes actions inside $\ldmlevelset$, its trajectory is guaranteed to stay within the $c$-thresholded support of the data distribution. Thus, using an LDM, we can verify that an agent can stay in-distribution for any desired density level. 

% An additional property of LDMs is that they can, in fact, exactly recover a Lyapunov function for the system under special cases of density models, or exactly recover its associated density model under special cases of dynamical systems. For a more detailed analysis of the connection between LDMs, density models, and CLFs, see Appendix \ref{appdx_connections_to_density_models}.

% In Fig. \ref{fig_density_example}, we illustrated an example where using a ``long-horizon'' density model as a constraint on the agent's actions could keep the agent in-distribution, where naively using a  density model constraint failed. The ``long-horizon'' density model in the example is in fact an LDM for the described system and data distribution. We can make the following formal statement as well:
% \begin{corollary}[LDM behavior on Example \ref{ex:mini_baseline_failure}]\label{ex:mini_baseline_failure_ldm}Theorem~\ref{thm:ldm_base_guarantee} implies that using an LDM constraint in the Example \ref{ex:mini_baseline_failure} setup yields actions with $P(s_t, a_t) \geq \frac{1}{4} \; \forall t \geq 0$.
% \end{corollary}

% For a given dynamical system and probability distribution, there could be many different functions that satisfy the conditions in Definition~\ref{ldm-definition} that characterize an LDM. 
We define a notion of a maximal LDM to represent the LDM with the largest state-action control-invariant sets for a given dynamical system and probability distribution. In our later discussion on deriving control policies with LDMs, we make use of maximal LDMs in order to impose the least restriction on the downstream task-solving policy.  
% While all valid LDMs can be used to make guarantees about an agent's ability to stay in-distribution, some make weaker guarantees than others. Let $\ldm_1(\st, \at)$ and $\ldm_2(\st, \at)$ be two different valid LDMs for the same dynamical system and data distribution. By Theorem~\ref{main-theorem} we know that the sublevel sets of $\ldm_1(\st, \at)$ and $\ldm_2(\st, \at)$ at $-\log(c)$, $\mathcal{G}_{1, c}$ and $\mathcal{G}_{2, c}$, both guarantee an agent inside the set is able to satisfy $P(\st, \at) \geq c$ for the duration of its trajectory. However, the elements inside $\mathcal{G}_{1, c}$ and $\mathcal{G}_{2, c}$ might be different. The larger of the two sets is able to make the $p$-threshold guarantee for a greater number of state and actions, thereby making a more useful statement than the smaller set. In order to make the strongest guarantees allowable by the dynamical system and data distribution, we define a notion of a maximal LDM to represent the LDM with the largest state-action control-invariant sets.

\begin{definition}[Maximal LDMs] An LDM $\ldm(\st, \at)$ is maximal if its each sublevel set $\ldmlevelset\!=\!\{(\st, \at)|\ldm(\st, \at)\!\leq\!-\log(c)\}$ is the largest state-action control-invariant set contained inside the corresponding sub-level set of the data distribution, $\mathcal{D}_{c}\!=\!\{(\st,\!\at)|P(\st,\!\at)\!\geq\!c\}$.
\label{maximal-ldm-definition}
\end{definition}

% In fact, because the maximal LDM is uniquely determined by a given data distribution and dynamics model, it can be represented in closed form with the following equation:
The maximal LDM can be represented in closed form with the following equation:
\begin{equation}
\ldm(s_0, a_0) = \min_{\{a_t\}_{t=1}^\infty} \max_{t \geq 0} -\log(P(s_t, a_t)).
\label{mdp-formulation}
\end{equation}

\begin{proposition}\label{thm:max_ldm}
If $\ldm(s, a)$ defined by (\ref{mdp-formulation}) is continuous, then it is a maximal LDM for $f(s, a)$ and $P(s,a)$.
\label{tabular-ldm-thm}
\end{proposition}

%For an in-depth example of a 2D linear system that highlights the properties of an LDM, see Appendix\ref{appendix:sec4_example}.

% \subsection{Example: Two-dimensional Linear System}

Here, we introduce an illustrative example of LDMs for a 2D linear system, given by
\vspace{-1em}
\begin{equation}
\label{eq:linear}
    s_{t+1} = F s_{t} + G a_t,
% \vspace{-0.5em}
\end{equation}
where $F=e^{A \Delta t}$, $G=A^{-1}(e^{A \Delta t}-I)B$, with
$A = \begin{bmatrix} \beta & \omega \\ -\omega & \beta \end{bmatrix}, B = \begin{bmatrix} 0 \\ 1 \end{bmatrix}$, $s:=\begin{bmatrix}x_1 \\ x_2\end{bmatrix}$ and $\beta, \omega > 0$. %, $\Delta t = 2\pi / 12$, $\omega=1$, and $\beta=0.15$.
The system spirals outwards from the origin when no action is applied, but can be stabilized to the origin with an control input. 

% When no action is is applied, the system spirals outwards from the origin. However, the action can shift the state in the $x_2$--direction, thus, a policy stabilizing to the origin, for instance, the LQR feedback policy $\pi_{LQR}(s)$ with input saturation at $||a||=5$ (Fig. \ref{fig_linear_appendix1}) can be designed.

% a stabilizing policy can be used to stabilize the system by shifting the state in the $x_2$--direction.
% Intuitively, when no action is applied, for each timestep the system \eqref{eq:linear} undergoes a rotation of $\omega \Delta t$ in the clockwise direction and a scaling of $e^{\beta \Delta t}$ in the radial direction.
% The origin is unstable and the state will spirally diverge without any control, however, the action can shift the state in the $x_2$--direction, which can be used to stabilize the system. For linear systems like \eqref{eq:linear}, an LQR-based stabilizing controller can be designed as a feedback policy, $\pi_{LQR}(s)=-K s$. % Since $\beta>0$, the origin is an unstable focus and the state will spirally diverge. The action can shift the state in the $x_2$--direction, which can be used to stabilize the system. Finally, for linear systems like \eqref{eq:linear}, an LQR-based stabilizing controller can be designed by solving a discrete-time Riccati equation, which results in a feedback policy $\pi_{LQR}(s)=-K s$. 

\begin{figure}
\includegraphics[width=\columnwidth, trim={0 0 0 0},clip]{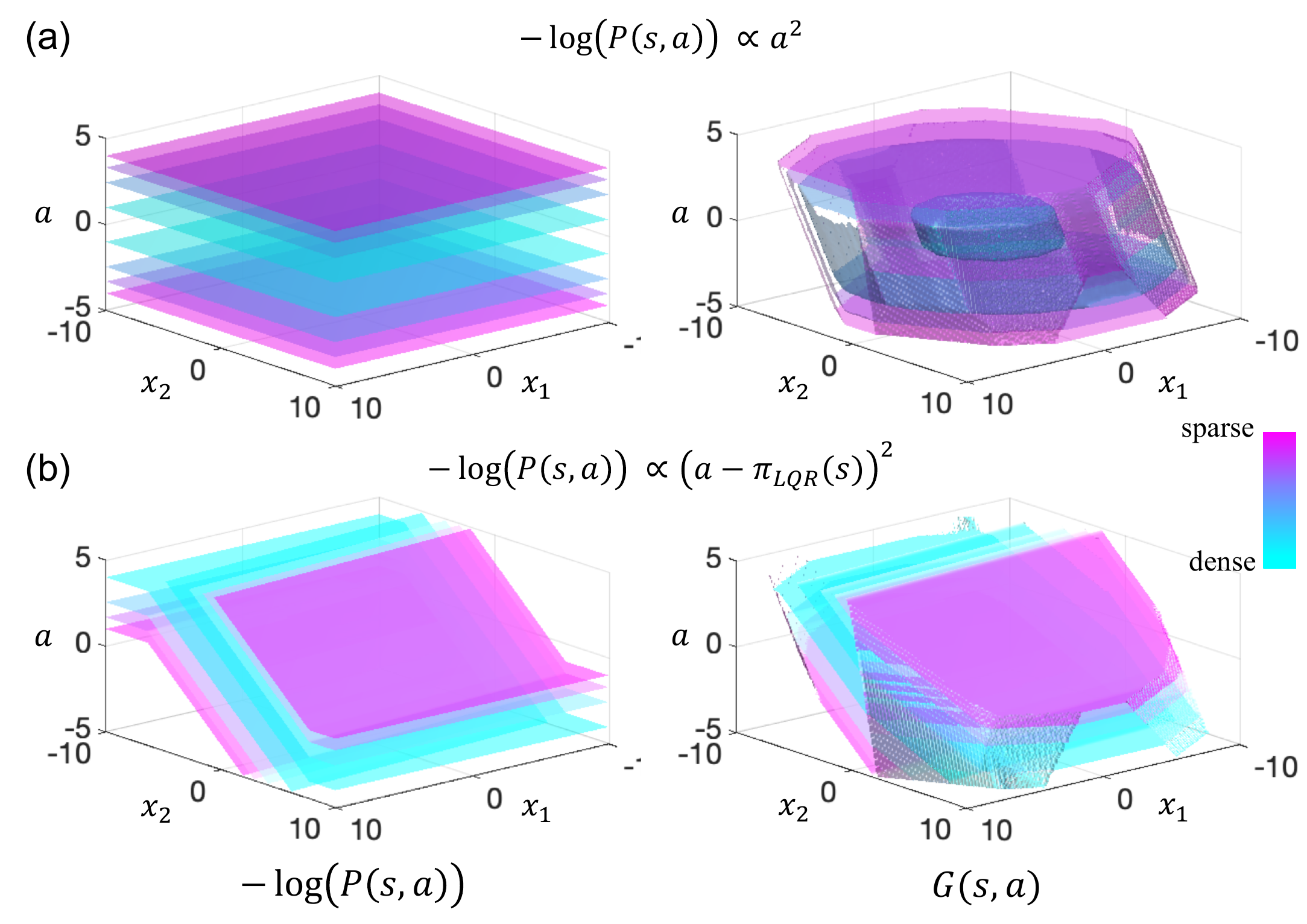}
\vspace{-1.5em}\caption{\footnotesize Level sets of various data distribution $P(s, a)$ (left) and the corresponding level sets of the maximal LDMs $G(s, a)$ (right) for the linear system example \eqref{eq:linear}. The color closer to cyan indicates denser region with lower $-\log{\left(P(s,a)\right)}$ or lower level of $G(s, a)$. (a) Gaussian policy with zero-mean and (b) Gaussian policy with mean $\pi_{LQR}(s)$ (stabilizing LQR controller).}\label{fig_linear_ldm}
\vspace{-1.5em}
\end{figure}

Figure \ref{fig_linear_ldm} shows example data distributions (left) and their associated maximal LDMs (right) for this system. 
% In Figure \ref{fig_linear}, level sets of the maximal LDMs $G(s, a)$ (on the right) that can be obtained for various distributions of the data density $P(s, a)$ (on the left) are visualized. The LDMs are computed on the grid of $(s, a)$ by using the method that will be described in the next section, so as to maximize the size of their state-action invariant sets. % The cyan color indicates denser region with lower $-\log{\left(P(s,a)\right)}$.
In particular, case (a) represents a scenario where we sample zero-mean Gaussian action uniformly across the bounded state space. % $s\in[-10, 10]\times[-10, 10]$. 
The data is dense near $a_t\!=\!0$, and sparse at higher action magnitudes. Since the natural dynamics is unstable, the system can only maintain a state that does not escape the high-density support near the origin. This results in the level set of $G(s, a)$ for a high density level (cyan) to be very small. % If we are allowed to explore lower-density region, then the corresponding invariant set can become much larger since now we are allowed to choose larger actions. 
In contrast, in case (b), the data distribution is collected by running a noisy feedback controller which stabilizes the system. %This represents a scenario where we have collected data from a controller that exhibits reasonably good behavior. 
In this case, the dense region of the data already inherits a stabilizing property. Therefore, it turns out that the invariant set for the high-density level maintains most volume of the original high-density region. Appendix \ref{appendix:sec4_example} provides more details and examples. % Finally, in case (c), a toric data distribution is given. Note that the natural dynamics of \eqref{eq:linear} induces symmetric rotations. The invariant sets on the left reveal that indeed it is possible to maintain the system to stay in a torus with an appropriate density level.

% To sum up, this example again highlights that to stay in-distribution, it is \textit{not} sufficient to maintain the density level only at the current time. Instead, it should be maintained to stay in a state-action control invariant set. Moreover, depending on the distribution of the data, the invariant set can be large or small, revealing the quality of the source of data. More details of this example including the policies associated with the LDMs are given in Appendix \ref{appendix:sec4_example}.

\vspace{-0.1in}
\section{Learning LDMs from Data}\label{sec:learning_ldms}

In order to make use of LDMs, we need a practical way of constructing them from data. In this section, we present an algorithm for learning a maximal LDM. 
First, to represent the distribution of our training data, we train a density model $E(\st, \at) = -\log(P(\st, \at))$. Next, we introduce the following LDM backup operator:
\begin{equation}\label{eq:bellman_op_density}
\mathcal{T}\ldm(s,a) = \max\{E(s, a), \min_{a' \in \mathcal{A}}  \gamma \ldm(f(s, a), a')\}    
\end{equation}
% This operator is derived from the definition of $G$ (provided in Appendix \ref{appendix:convergence}).
% This operator closely resembles the Bellman backup operator, except that instead of adding together the reward and the next ``value'' (in this case, $\ldm(f(s, a), a')$), it takes the maximum of the two. However, just like in Q-learning, the action $a'$ at the next time step is chosen to optimize the learned function, so the LDM $\ldm(s,a)$ plays a similar role to the Q-function in a Q-learning algorithm. Analogous to value iteration or Q-learning, where one can solve for the value function via dynamic programming by iteratively applying the Bellman backup, we can solve for the maximal LDM by iteratively applying the LDM backup. 
Analogous to value iteration or Q-learning (here, the LDM $\ldm(s,a)$ plays a similar role to the Q-function), where one can solve for the value function via dynamic programming by iteratively applying the Bellman backup, we can solve for the maximal LDM by iteratively applying the LDM backup. 
By initializing $\ldm(s, a)$ to $-\log(P(s, a))$ and performing the update rule $\ldm_{k+1} \leftarrow \mathcal{T}\ldm_k$, $\lim_{k \rightarrow \infty} \ldm_k$ will converge to the solution of Equation (\ref{mdp-formulation}) when $\gamma=1$ (a formal statement and proof are provided in Appendix \ref{appendix:convergence}).
While this method can easily be implemented without any discount, corresponding to $\gamma=1$, the use of a discount factor aids the theoretical analysis \footnote{Discounted iteration of Equation~\ref{eq:bellman_op_density} yields an LDM that satisfies $\ldm(s_0, a_0) = \min_{\{a_t\}_{t=1}^\infty} \max_{t \geq 0} - \gamma^t \log(P(s_t, a_t))$, which essentially counts densities far into the future as being lower than they really are, which we can intuitively treat as either as a heuristic variance reduction technique or a way to account for greater uncertainty far in the future. We additionally perform ablation experiments to analyze the effects of using a discount $\gamma < 1$ empirically (Appendix \ref{app:ablation_discount_factor}). }.

The above procedure makes use of the true dynamics $\stp = f(\st, \at)$. However, in reality we only have access to a dataset of transitions $D = \{(\st^i, \at^i, \stp^i)\}_{i=1}^N$, which may be used to learn a dynamics model or a policy for solving a downstream task. Thus, we would like to learn the LDM from this same dataset. We present a fitted version of LDM training, in which the learned LDM $\hat{G}_{k+1}$ is optimized using samples by iteratively applying:
\vspace{-0.5em}
\begin{align}\label{ldm_update}
\hat{G}_{k+1} \in \arg\min\limits_{G\in\mathcal{G}} \sum_{(s, a) \in D} \left(G(s, a) - \mathcal{T}\hat{G}_k(s, a)\right)^2
\vspace{-0.5em}
\end{align}
We analyze this procedure formally in Sec. \ref{sec:theory}.
% We summarize this abstract algorithm in Alg. \ref{ldm_alg_theoretical}, and analyze it formally in Sec. \ref{sec:theory}.
% \begin{algorithm}[H]
% \SetAlgoLined
%  input parameters: function class $\mathcal{G}$, \# of training steps $K$, discount parameter $\gamma$\\
%  train density model $E(s, a) = -\log P(s,a)$\\
%  set $\hat{G}_0(s, a) \leftarrow E(s,a)$ \\
%  \For{k=0, \ldots, K-1}{
%   $\hat{G}_{k+1} \in \arg\min\limits_{G\in\mathcal{G}} \errtext{\sum_{(s, a) \in D} \left(G(s, a) - \mathcal{T}\hat{G}_k(s, a)\right)^2}{$L_D(G, \hat{G}_k)$}$

%  }
%  return $\hat{G}_K$
%  \caption{LDM Training [abstract algorithm]}
% \label{ldm_alg_theoretical}
% \end{algorithm}

Finally, in order to use LDMs in high-dimensional and realistic domains, we present a practical algorithm for training LDMs with deep neural networks. We represent the density model $E_\theta(\st, \at) = -\log(P(\st, \at))$ with a neural network, and train it as a flow model \cite{durkan2019neural} (implementation details in Appendix \ref{appendix:ldm_training}). We also train a neural network model of the LDM $\ldm_{\phi}(\st, \at)$ in conjunction with a policy $\pi_{\psi}(\st)$. This algorithm is structurally similar to standard off-policy actor-critic RL methods, and is implemented in a similar way in practice by modifying standard deep RL implementations, replacing the standard Bellman backup with our \textcolor{byzantine}{LDM backup} (Eq. \ref{eq:bellman_op_density}). The policy is optimized to output the actions that minimize the LDM (Eq. \ref{policy_objective}), and the policy’s actions are used to perform the LDM backup (Eq. \ref{ldm_objective}), because exactly solving for the minimizing action at each step is difficult with continuous, high-dimensional action spaces:\footnote{In our theoretical analysis, we will still assume that the minimization of the LDM with respect to $a$ can be done exactly, since the policy has unrestricted access to $\mathcal{T}\hat{G}_{k}$.}
\begin{align}
\psi =& \argmin_\psi \mathbb{E}_{\st \sim p_D}[\ldm_{\phi}(\st, \pi_{\psi}(\st))] \label{policy_objective}\\
\phi =& \argmin_\phi \mathbb{E}_{\st, \at, \stp \sim p_D}[(\ldm_{\phi}(\st, \at) - \nonumber \\
&\textcolor{byzantine}{\max\{E_\theta(\st, \at), \gamma G_\phi(\stp, \pi_\psi(\stp))\}})^2] \label{ldm_objective}
\end{align}
%%SL.1.23: make sure to include all the nitty-gritty implementation details (e.g., what kind of flow model) in an appendix somewhere
% Following the implementation in SAC~\citep{haarnoja2018soft}, we also make use of a number of deep RL techniques to stabilize training. These include using double value networks $\phi_1, \phi_2$, target networks $\bar{\phi}_1, \bar{\phi}_2$, and an additional entropy term in the training objective with automatic tuning, $\alpha \mathcal{H}(\pi_\psi(.|\st))$. 
Following the implementation in soft-actor critic (SAC)~\citep{haarnoja2018soft}, we also make use of a number of deep RL techniques to stabilize training. These include using double value networks $\phi_1, \phi_2$, target networks $\bar{\phi}_1, \bar{\phi}_2$, and an additional entropy term in the training objective with automatic tuning, $\alpha \mathcal{H}(\pi_\psi(.|\st))$. Furthermore, because we are learning the LDM from an offline dataset, we follow the implementation in CQL\cite{kumar2020conservative} by adding a conservative regularization term, CQL($\mathcal{H}$), in the LDM update to regulate underestimation \footnote{Since Q-values are \emph{maximized} while LDMs are \emph{minimized}, our conservative term handles \emph{underestimation}, not \emph{overestimation}.} in the LDM on out-of-distribution actions.
%%SL.1.23: note that some readers might not realize why it's underestimation vs overestimation (ie miss the sign flip), perhaps it would be good to briefly explain this
% Our algorithm uses the loss functions: 
% \begin{align}
% \label{critic_update}
% L_\ldm(\phi) = &\mathbb{E}_{\st, \at, \stp \sim p_D}[(\ldm_{\phi}(\st, \at) - G_{\text{tar}})^2+\text{CQL}(\mathcal{H})] \nonumber\\
% G_{\text{tar}} = &\max\{E_\theta(\st, \at), \gamma G_{t+1}\}\nonumber\\
% G_{t+1} = & \max\{\ldm_{\bar{\phi}_1}(\stp, \pi_{\psi}(\stp)), \ldm_{\bar{\phi}_2}(\stp, \pi_{\psi}(\stp))\} \nonumber\\
% & - \alpha \mathcal{H}(\pi_\psi(.|\stp))] \nonumber\\
% L_\pi(\psi) =& \mathbb{E}_{\st \sim p_D}[\ldm_{\phi}(\st, \pi_{\psi}(\st))] - \alpha \mathcal{H}(\pi_\psi(.|\st)) \nonumber\\
% L(\alpha)=& \alpha (\mathcal{H}(\pi_\psi(.|\st)) - H_{\text{tar}}) \nonumber
% %\label{actor_update}
% \end{align}
A summary of our practical algorithm can be found in Alg. \ref{ldm_alg}. We make use of this algorithm in our experiments in Section \ref{sec:experiments}. More implementation details can be found in Appendix \ref{appendix:ldm_training}.

% \begin{algorithm*}%[H]
% \SetAlgoLined
%  Initialize parameter vectors $\theta$, $\phi_1$, $\phi_2$, $\bar{\phi}_1$, $\bar{\phi}_2$, $\psi$, and $\alpha$\\
%  Train flow model $E_\theta (\st, \at)$\\
%  \For{\text{num LDM training steps}}{
%   Update LDM networks: \\
%   $\text{ }\text{ }\text{ } \phi_i \gets \phi_i - \lambda_\ldm \nabla_{\phi_i} L_\ldm(\phi_i) $ for $i \in \{1, 2\}$\\
%   Update LDM target networks:\\
%   $\text{ }\text{ }\text{ } \bar{\phi}_i \gets \tau \phi_i + (1-\tau) \bar{\phi}_i$ for $i \in \{1, 2\}$\\
%   Update policy network:\\
%   $\text{ }\text{ }\text{ } \psi \gets \psi - \lambda_\pi \nabla_{\psi} L_\pi(\psi) $\\
%   Update entropy multiplier:\\
%   $\text{ }\text{ }\text{ } \alpha \gets \alpha - \lambda_\alpha \nabla_{\alpha} L(\alpha)$

%  }
\begin{algorithm*}%[H]
\SetAlgoLined
 Initialize parameter vectors $\theta$, $\phi_1$, $\phi_2$, $\bar{\phi}_1$, $\bar{\phi}_2$, $\psi$, and $\alpha$\\
 Train flow model $E_\theta (\st, \at)$\\
 \For{\text{num LDM training steps}}{
  Take gradient step on LDM networks: \\
  $\text{ }\text{ }\text{ }\text{ } \phi_i \gets \phi_i - \lambda_\ldm \nabla_{\phi_i} \Big(\mathbb{E}_{\st, \at, \stp \sim p_D}[(\ldm_{\phi_i}(\st, \at) - \max\{E_\theta(\st, \at), \gamma G_{t+1}\})^2+\text{CQL}(\mathcal{H})]\Big) $ for $i \in \{1, 2\}$\\
  $\text{ }\text{ }\text{ }\text{ }  \text{where } G_{t+1} = \max\{\ldm_{\bar{\phi}_1}(\stp, \pi_{\psi}(\stp)), \ldm_{\bar{\phi}_2}(\stp, \pi_{\psi}(\stp))\}$\\
  Update LDM target networks:\\
  $\text{ }\text{ }\text{ }\text{ } \bar{\phi}_i \gets \tau \phi_i + (1-\tau) \bar{\phi}_i$ for $i \in \{1, 2\}$\\
  Take gradient step on policy network:\\
  $\text{ }\text{ }\text{ }\text{ } \psi \gets \psi - \lambda_\pi \nabla_{\psi} \Big( \mathbb{E}_{\st \sim p_D}[\ldm_{\phi}(\st, \pi_{\psi}(\st))] - \alpha \mathcal{H}(\pi_\psi(.|\st)) \Big) $\\
  Take gradient step on entropy coefficient:\\
  $\text{ }\text{ }\text{ }\text{ } \alpha \gets \alpha - \lambda_\alpha \nabla_{\alpha} \Big(\alpha (\mathcal{H}(\pi_\psi(.|\st)) - H_{\text{tar}})\Big)$
 }
 \caption{LDM Training [practical algorithm]}
\label{ldm_alg}
\end{algorithm*}

\section{Control with LDMs}\label{sec:ldm_control}
In this section, we discuss how LDMs can be used with learned models to produce control strategies that remain in-distribution while accomplishing a task. We focus on applying LDMs to model-based reinforcement learning, a problem setting that is known to be particularly sensitive to distribution shift, though in principle LDMs could be used with other learned models, including model-free or imitation learning policies, all of which are susceptible to distributional shift. Model-based RL methods can use learned models to train policies, or use learned models directly for control via planning (e.g., via model-predictive control, MPC). We focus on the latter category of methods to provide a simple evaluation of LDMs, which is amenable to both theoretical and empirical analysis.

%%SL.1.27: merged some of the materials into the above para and generally shortened
%In reinforcement learning, we are given a dataset of system transitions \mbox{$D = \{(\st^i, \at^i, \stp^i)\}_{i=1}^N$} and a reward signal $r(\st, \at)$, and the goal is to devise a policy which maximizes the cumulative reward. Model-based algorithms are a class of RL algorithms that uses a learned model to derive a policy, where the specific mechanism for obtaining the policy vary greatly across different algorithms. In this work, we consider one of the most simple approaches: using the learned model for planning via model predictive control. We choose to study the effects of LDM on distribution shift under this framework because it is well studied in the controls literature and amendable to theoretical analysis. 

\subsection{MPC with an LDM Constraint}\label{sec:mpc_ldm}

In our analysis, we will apply the LDM to constrain a short-horizon MPC controller that uses a learned model, such that the controller stays close to the training distribution where the model is likely to be accurate. The model $f_{\xi}(\st, \at)$ is trained on the same dataset as the LDM, which we assume is sampled from $P(\st,\at)$, using empirical risk minimization on a sample-based estimate of the loss $L_f(\xi)$:%\vspace{-1em}
\begin{equation}
\label{eq:model_erm_loss}
L_f(\xi) = \mathbb{E}_{(\st, \at, \stp) \sim p_D}[(f_{\xi}(\st, \at) - \stp)^2].
\end{equation}
This model can then be used with MPC to select actions $\mathbf{a}_{1:H}$. The LDM, trained via Alg.~\ref{ldm_alg}, modifies the standard MPC procedure by imposing an additional constraint:
\begin{align}\label{mbrl_opt_prob}
{a}_{1:H}^* = &\argmax_{{a}_{1:H}} \sum_{t=1}^H r(\st, \at) \\
    \text{s.t.} \quad & \stp = f_{\xi}(\st, \at) \quad \forall 1 \leq t \leq H-1 \nonumber \\
    & \ldm_{\phi}(\st, \at) \leq -\log(c) \quad \forall 1 \leq t \leq H
    \label{ldm_constraint}
\end{align}
As in standard MPC, the controller executes the first action, and then replans. Theorem \ref{main-theorem} guarantees that, for a correct LDM, as long as the agent is in the sub-level set of the LDM at $-\log(c)$, it is able to satisfy $P(\st, \at) \geq c$ throughout its trajectory. Thus, under a correct LDM on Line (\ref{ldm_constraint}), this method of constraining MPC is guaranteed to produce trajectories in the $c$-thresholded support of the data.
% Further, from any starting state $s_0$ we can actually compute the maximal threshold $c_{\mathrm{max}}(s_0) \doteq e^{-\inf_{a_0} G(s_0,a_0)}$ such that there is a sequence of actions $\{a_t\}_{t\geq0}$ for which $P(s_t, a_t) \geq c_{\mathrm{max}}(s_0)$ for all $t\geq 0$. The usefulness of this quantity is twofold: (1) it can serve as a safety metric for the state space, (2) it tells us specifically which $c$ values ensure that the MPC optimization problem with constraint \ref{ldm_constraint} is feasible (i.e. has a solution).

Finally, since Algorithm \ref{ldm_alg} aims to learn the maximal LDM, it is the \emph{least restrictive} constraint for making the guarantee, giving the agent the most flexibility for solving the task. Of course, in practice the LDM itself might suffer from approximation errors; we analyze these further in Section~\ref{sec:theory}.

\subsection{MPC with a Density Model Constraint}
% In Example \ref{ex:mini_baseline_failure}, we discussed a setting where imposing a na\"{i}ve density constraint on the control for only the current step could lead to inevitable failure in future steps. However, one may wonder if imposing the density constraint for a longer horizon, i.e. solving the MPC optimization problem with 
One may wonder if imposing the density constraint for a longer horizon by planning with the constraint
\begin{equation}\label{density_constraint}
E_{\theta}(\st, \at) \leq -\log(c) \quad \forall 1 \leq t \leq H\end{equation} in place of Equation (\ref{ldm_constraint}) would alleviate the myopic behavior shown in Figure \ref{fig_density_example}. We find that the answer is \textit{no}. We construct an example scenario, extending the one in Figure \ref{fig_density_example}, for which planning with with a density model constraint fails even under perfect knowledge of the dynamics and data distribution, but planning with an LDM can succeed.
\begin{example}\label{ex:baseline_failure}Let $\mathcal{S} = \mathcal{A} = \mathbb{Z}$, $f(s,a) = s+a$, $r(s,a) = a$ and $s_0 = 0$. For any $H \geq 1, \epsilon > 0$, $\exists$ data distribution $\mathcal{D}_{H, \epsilon}$ for which planning with constraint (\ref{density_constraint}) will lead to $s_{H+1}$ s.t. $P(s_{H+1}, a) \leq \epsilon \; \forall a\in\mathcal{A}$. However, constraining via the LDM (\ref{ldm_constraint}) ensures we only reach states $s_t$ for which $\exists a_t$ with $P(s_t, a_t) \geq \frac{1}{2(H+1)}$.
\end{example}
Appendix~\ref{appendix:ldm_control} contains the explicit construction, a discussion of its interpretation, and the proof of the claims in Ex.~\ref{ex:baseline_failure}. The failure mode of (\ref{density_constraint}) highlighted above depends on the horizon of the planning problem. This is because larger $H$ enables MPC to plan further into the future, and hence satisfy the density lower bound over more time steps. However, in practice one cannot set $H$ to be too large because exceedingly large horizons are computationally expensive, lead to accumulated model error, and potentially increase variance~\citep{tassa2012,Deisenroth2015}. Current model-based RL methods that use MPC typically utilize horizons as low as $H=10$ or $H=20$~\citep{chua2018deep}, $150\times$ lower than the full task length.
%%SL.1.27: shortened this a bunch to make room
%For example, \citet{Deisenroth2015} address the inability of deterministic methods to cope with extra bias introduced by horizons that are not very short, PETS \cite{chua2018deep} is robust to this issue and can use less short planning horizons, but the longest considered is still $150\times$ smaller than the task-trajectory. Finally, when acting online in real-time, speed is of utmost importance and hence a short horizon is a practical necessity \cite{tassa2012}.

The main takeaway is that density thresholding fails even with the true dynamics $f$ and distribution $P$. % despite there being a clear ``correct'' solution.
% We see from the construction that the problem comes from taking an action that, while myopically good, takes us down a path which will become unstable/unknown. Again, we see that the LDM solves this issue as it filters out actions for which there is an unacceptable probability degradation at any time in the future (Thm.~\ref{thm:ldm_base_guarantee}). As a result, 
Using an LDM constraint, we can plan with any horizon and ensure that we will stay within some reasonably well-visited region (and hence we can support arbitrarily long task-horizon).

\section{Theoretical Analysis} \label{sec:theory}
So far our discussion has assumed that the LDM is learned perfectly, but this is not the case in practical settings due to approximation error. We analyze the effect of such errors on the LDM next. We aim to answer the following two questions: (1) Can we bound the error in the LDM when it is trained from data? (2) How does this error influence the ability of the LDM to keep MPC procedures (Eq. \ref{mbrl_opt_prob}) in-distribution? Our analysis will show that even with approximation error, the LDM retains provable guarantees for keeping MPC within high-density regions, in contrast to the na\"{i}ve constraint based on the density model. We conclude the section by returning to Example~\ref{ex:baseline_failure} and showing that the learned LDM (via Eq. \ref{ldm_update}) is still able to remedy the myopic tendencies of naive density thresholding.
%For both of these, we get positive results. In Sec.~\ref{subsec:ldm_learn} we derive the error bound of Alg.~\ref{ldm_alg_theoretical} 
%under a set of reasonable assumptions, and in Sec.~\ref{sec:mpc_theory} we show that using an approximate LDM as a constraint for MPC retains provable guarantees on keeping the agent in high-density regions long-term (and hence in regions where any data-based model is reliable). This is in contrast to a na\"{i}ve constraint based on only the density model.
\subsection{Error Bound of LDM Learning Procedure}\label{subsec:ldm_learn}

The goal of this section is to give a formal guarantee for repeatedly performing the update in Eq. \ref{ldm_update}. One cause for error in our updates is the lack of access to the true data distribution. What we will ultimately converge to is the fixed point of Eq. \ref{eq:bellman_op_density}, the $\gamma$-discounted LDM using the density model $E(s,a)$: $\hat{G}^\star \doteq \min\limits_{\{a_t\}_{t=1}^\infty} \max\limits_{t \geq 0} \gamma^t E(s_t, a_t)$. Hence, our main result will analyze convergence to this quantity. To do so, we first introduce a quantity that captures the maximal level of `recoverability' of a system, i.e., how many times larger the maximal reachable density from any given initial conditions can be.
\begin{definition}[Recoverability]Define the recoverability of the system to be\begin{equation*}
R\doteq \sup_{\substack{{s_0, T\geq 0, \{a_t\}_{t=0}^T}\\\text{s.t. }P(s_0, a_0) > 0}} \frac{P(s_T, a_T)}{P(s_0, a_0)} \;\; \text{s.t.}\;s_{t+1} = f(s_t, a_t).
\end{equation*}\label{recoverability_def}\end{definition}
Intuitively, a small $R$ would mean that once we reach a low density region, we cannot hope to go back to a high density region, while a large $R$ means that we can considerably multiply the density of any initial condition by acting appropriately. We define the $P$-norm of a function $\|g\|_{P} \doteq \mathbb{E}_{(s,a) \sim P}|g(s,a)|$ and its $\infty$-norm $\|g\|_{\infty} \doteq \sup_{s,a}|g(s,a)|$. The main result (proved in Appendix~\ref{appendix:ldm_control}) is stated below:
\begin{proposition}\label{prop:ldm_fqi} Let $\hat{G}_0 = E$. The result of applying the update in Eq. \ref{ldm_update} $K$ times satisfies:
\begin{equation}
\|\hat{G}_K - \hat{G}^\star\|_P \leq \dfrac{R \cdot \epsilon_{\mathrm{ls}}}{1-\gamma} + \gamma^K \cdot \|E - \hat{G}^\star\|_{\infty}, \label{eq:fqi_conv}
\end{equation}
where $\epsilon_{\mathrm{ls}} \doteq \max\limits_{t\in[K-1]}\|\hat{G}_{t+1} - \mathcal{T}\hat{G}_t\|_P$.
\end{proposition}
Proposition \ref{prop:ldm_fqi} quantifies the approximation error from executing the update (Eq. \ref{ldm_update}) and its dependence on the recoverability factor ($R$), discount ($\gamma$), \# of iterations ($K$), $\|E - \hat{G}^\star\|_{\infty}$, and the generalization error of the least squares fit ($\epsilon_{\mathrm{ls}}$). The quantity $\epsilon_{\mathrm{ls}}$ is simply the accuracy of fit under a given distribution, and it is standard to assume that this decreases as the dataset becomes larger. % To our knowledge, all prior RL analyses of fitted iteration rely on such a quantity. 
Bounds on this quantity fall under the study of supervised learning, and we state a standard bound in Appendix~\ref{appendix:ldm_control}, along with a corollary regarding its implication for Proposition~\ref{prop:ldm_fqi}. 

Below, we give two additional remarks which may capture the convergence behavior of the LDM learning procedure more tightly under certain situations. Firstly, in cases where $R$ is large, we can bound error in terms of the \textit{one-step} recoverability parameter $r \doteq \sup_{\st,\at, a_{t+1}} \{ \frac{P(f(\st,\at), \atp)}{P(\st, \at)} \text{ s.t. } P(\st,\at) > 0 \}$. This captures how much the density can decrease in a single step, which may be orders of magnitude smaller than $R$.

\begin{remark}\label{rmk:fqi_conv_2}
$\hat{G}_K$ also satisfies $\|\hat{G}_K - \hat{G}^\star\|_P \leq \frac{\epsilon_{\mathrm{ls}}}{1-r\gamma} + (r\gamma)^K \cdot \|E - \hat{G}^\star\|_{P}$ for $\gamma < r^{-1}$.
\end{remark}

% Secondly, since we have control over $\gamma$, depending on the tradeoff between $r$ and $R$, different approaches to choosing $\gamma$ may be better. We require $\gamma <1$ to ensure decay of $\|E - \hat{G}^\star\|_{\infty}$'s contribution. However, as we discuss further in Appendix~\ref{app:theory}, there are scenarios under which we may set $\gamma=1$:
Secondly, there are scenarios (such as when the system is stable) under which we may set $\gamma\!=\!1$.

\begin{remark}\label{rmk:gamma_1}
%If $\exists K_{\mathrm{fin}}$ such that repeated application of $f$ $\geq K_{\mathrm{fin}}$ times 
If $\exists K_{\mathrm{fin}}$ such that the system evolving more than $K_{\mathrm{fin}}$ times leads to states where the difference between the density model and the LDM value becomes negligible ($\leq \epsilon_{\mathrm{fin}}$), we can set $\gamma = 1$ and satisfy $\|\hat{G}_K - \hat{G}^\star\|_P \leqslant R\cdot K_{\mathrm{fin}}\cdot \epsilon_{\mathrm{ls}} + \epsilon_{\mathrm{fin}}$ where $\hat{G}^\star$ is the undiscounted LDM.
\end{remark}
% This can happen if our dataset consists of trajectories which terminate in different parts of the space (in fact, the construction in Example~\ref{ex:baseline_failure} satisfies this with $K_{\mathrm{fin}} = H+2$), or if the system is stable.
%%SL.5.13: I'm not *sure* about this, but I wonder if it could make sense to have some remark about how many of the tasks in our experiments *do* terminate in different parts of the space, as a kind of vague excuse for using gamma=1 in our experiments? Though maybe that's unnecessary... or even worse, unconvincing.
%%PG.5.26: Stable applies to simglucose and vent. Katie -- does the diff part thing apply to lunar lander & hopper? 

As a final note, the LDM as formulated above addresses fully observed settings, but in Appendix~\ref{app:partial} we provide the natural extension of the LDM to
non-Markovian/partial state observations. In fact, the \emph{same} procedure as
before still provides a distributional shift guarantee, but with an additional error term that accounts for the variability in the
observation for the same underlying state (Prop.~\ref{prop:partial}). 

\subsection{Error Sensitivity of the LDM Barrier for MPC}\label{sec:mpc_theory}

In this section we show that, even in the presence of approximation and discount error, a learned LDM constraint (\ref{ldm_constraint}) retains provable guarantees on staying in high density regions, which means that MPC with an LDM constraint will not query the model $\hat{f}$ on inputs where it was not trained.
%This implies that using an LDM for MPC successfully and realistically enables us to: 1) produce behavior which is not too divergent from the one in the reference dataset, and 2) only query the model $\hat{f}$ in regions with at least moderate density (which is essential for accuracy when $\hat{f}$ is also learned from data). %This is not immediately obvious, since although the true LDM clearly improves MPC, it is not necessarily clear that the errors in a learned LDM will still result in an overall improvement over the na\"{i}ve density constraint baseline.
There are three main sources of approximation error: error in the density model, error from training the LDM, and error from imposing a discount. %\footnote{The use of a discount in practice also means that, to ensure that the original undiscounted system stays in-distribution, we must slightly decrease the negative log threshold to compensate for this source of bias.}.
For cleaner presentation, we assume here the density model satisfies a pointwise error bound: $\exists \epsilon_p > 0$ s.t.  $|\log P(s,a) + E(s,a)| \leq \epsilon_p$. 
% This assumption simply says that the estimated (unconditional) densities are accurate on the majority of $\mathcal{S}\times \mathcal{A}$ -- a reasonable assumption if our density model trains and generalizes well. 
In Appendix~\ref{subsec:ldm_barrier_proofs}, all results are restated assuming a probably approximately correct error instead. %is weaker than a similar error bound on a dynamics model (for learned models we do not expect the error to be uniform across $\mathcal{S}\times\mathcal{A}$) since it 
Our result follows:
\begin{proposition}\label{prop:ldm_degradation}For any probability mass function $P$ and initial state $s_0$, taking action $a_0$ with $\hat{G}(s_0, a_0) \leq -\log{c}$ guarantees that $\forall t$, there is a sequence $a_{1:t}$ for which
\begin{equation*}
\log P(s_t, a_t) \geq \gamma^{-t} \log{c} - \frac{\gamma^{-t} R \, \epsilon_{\mathrm{ls}} \, \exp{\epsilon_p}}{c(1-\gamma)} - \epsilon_P
\end{equation*}
as $K\rightarrow \infty$. Further, if the assumption in Remark~\ref{rmk:gamma_1} holds, using $\gamma=1$ ensures
\begin{equation*}
\log P(s_t, a_t) \geq \log{c} - \frac{(R K_{\mathrm{fin}} \epsilon_{\mathrm{ls}} +\epsilon_{\mathrm{fin}})\exp{\epsilon_p}}{c} - \epsilon_P  
\end{equation*}
\end{proposition}

The $\gamma^{-t}$ factor in the bound above suggests that distributional shift is harder to avoid further in the future. However, if we are instead interested in the difference between the sum of discounted rewards the agent planned for and the one the agent actually experience, then we get a guarantee which does not decay exponentially in time. Let us denote the `planned-for' discounted cumulative reward for a state-action sequence $\mathbf{s} = s_{0:T}, \mathbf{a} = a_{0:T}$ by $\hat{R}_T(\mathbf{s}, \mathbf{a})\doteq \sum_{t=0}^{T-1} \gamma^t r(\hat{f}(s_t, a_t), a_{t+1})$ and the `true' reward $R_T(\mathbf{s}, \mathbf{a}) \doteq \sum_{t=0}^{T-1} \gamma^t r(f(s_t, a_t), a_{t+1})$. Assuming that the composition of the reward with the model has error which satisfies $\sup_{a'} |r(f(s,a), a') - r(\hat{f}(s,a), a')| \leq \frac{\epsilon_r}{\sqrt{P(s,a)}}$ for some $\epsilon_r \geq 0$ (which encompasses a $\sim P(s,a)^{-0.5}$ model error together with the Lipschitz property of the reward function), we derive the following corollary from Proposition~\ref{prop:ldm_degradation}:

\begin{corollary}\label{prop:reward_ldm_degradation}For any starting state $s_0$, taking action $a_0$ with $\hat{G}(s_0, a_0) \leq -\log{c}$ with $c\geq \frac{(1-\gamma) + R \epsilon_{\mathrm{ls}}\exp{\epsilon_P}}{(1 - \gamma^{T-1} (\epsilon_P + 2\log{\epsilon_r})) (1-\gamma)}$ guarantees that there is a sequence $a_{1:T}$ that yields $s_{1:T}$ satisfying
\begin{align*}
|\hat{R}_T(\mathbf{s}, \mathbf{a}) - R_T(\mathbf{s}, \mathbf{a})| \leq & \;(1 + \epsilon_P + 2\log{\epsilon_r}) \cdot \frac{1 - \gamma^T}{1-\gamma} \\
&+ T \cdot \left(\log{\frac{1}{c}} + \frac{R \epsilon_{\mathrm{ls}} \exp{\epsilon_P}}{c(1-\gamma)}\right) 
\end{align*}
\end{corollary}

% The lower bound on $c$ is necessary to relate the bound on $\log{P(s_t, a_t)}$ with the $\sim \frac{1}{\sqrt{P(s_t, a_t)}}$ error scaling. The stated lower bound requirement is a simplified version of its full counterpart, deferred to the Appendix (Eq. \ref{full_lb_reward_theory}). Note that (1) the current lower bound is very small, going to $0$ as $\epsilon_r\rightarrow 0$ or $\gamma\rightarrow 1$, and (2) there is a slack due to the application of 2 inequalities -- in practice the stated guarantee holds for even smaller settings of $c$.

We see that the first term captures the accumulation of density error over the trajectory and the second term captures how close we can get to this upper bound as a function of the horizon $T$, the threshold $c$, the discount $\gamma$, and the model errors. For example, if we have $\epsilon_\mathrm{ls} \approx 0$ and $c\rightarrow 1$, the difference would be dominated by $\approx (1+\epsilon_p) \frac{1- \gamma^{T}}{1-\gamma}$. By bounding the deviation of the real world discounted reward from the discounted reward expected by the controller, Corollary~\ref{prop:reward_ldm_degradation} translates the LDM guarantee into a bound on how ``over optimistic'' the model-based trajectory optimizer can possibly be. Note also that the guarantee applies every time we commit to an action (for $T \gg H$ potentially), and is refreshed/extended under iterative replanning.

% \subsection{LDM vs. Density Models for MPC}
\subsection{LDM vs. Density Model for MPC}
In Example \ref{ex:baseline_failure}, we constructed a situation where constraining MPC with an oracle density model of the data distribution can lead to myopic behavior, whereas the true LDM is able to prevent it. A natural question is whether the learnt (hence imperfect) LDM is still capable of eliciting the desired behavior? The answer is \textit{yes}, made precise below:

\begin{example}[Continuation of Example~\ref{ex:baseline_failure}]\label{cor:ex_7.1_w_ldm_error}In the setting of Prop~\ref{ex:baseline_failure}, if MPC with LDM is feasible for $\hat{c} \geq (R(H+2) \epsilon_{\mathrm{ls}} + 2\epsilon_p) \exp{\epsilon_p} + \epsilon$, we are guaranteed to only reach states with $P(s_t, a_t) \geq \frac{1}{2(H+1)}$. For small $\epsilon_p$, we only need $\hat{c} \gtrapprox R(H+2) \epsilon_{\mathrm{ls}} + \epsilon$. \end{example}

Intuitively, if $\epsilon_p, \epsilon_{\mathrm{ls}}$ are sufficiently small, we can expect the approximate LDM to succeed in filtering out the path which leads to dangerously low density regions, and, in this example, stay in regions with probability even higher than the formal guarantee. Example \ref{cor:ex_7.1_w_ldm_error} is an exemplification of how Prop. \ref{prop:ldm_degradation} shows that the LDM always ensures bounded distributional shift, solving the underlying horizon-dependent vulnerability of the na\"{i}ve density constraint to situations where violation after $H$ is inevitable if the wrong action is taken at the current step. 

%Hence, even in the presence of training error, the LDM constraint would allow us to pick the correct path from the very beginning. 

%Going back to our main guarantee formulation, Proposition~\ref{prop:ldm_degradation} bounds the lowest probability state-action tuple that the agent will observe if it uses the LDM constraint, even taking into account errors in LDM training. Beyond the ability to bound the reward estimate deviation explored previously, another direct implication is that an approximate LDM constraint (if feasible) will ensure we \textit{never} reach a region that is completely absent from the reference dataset. A consequence of this guarantee is that we are guarded from the myopic failure mode of density thresholding. This can be illustrated by going back to the setting of Prop~\ref{ex:baseline_failure} and exemplifying how using the LDM constraint would allow us to pick the correct path from the very beginning:

For a more general comparison of the approximate LDM and the density model constraints, we derive an analogue to Props.~\ref{prop:ldm_degradation} \& \ref{prop:reward_ldm_degradation} for the density threshold within the horizon $H$. In contrast to the purely statistical LDM analysis, providing guarantees for density thresholding requires additional smoothness assumptions on $E$ and $f$. In Appendix \ref{subsec:ldm_barrier_proofs}, we instantiate these assumptions and give the corresponding guarantees for density model constrained MPC.

%An immediate obstacle is that the best that planning with a density constraint can guarantee for $\tau < H$ is that there exist a sequence of $a_t$ such that $E(\hat{s}_\tau, a_\tau) \leq -\log{c}$ where $\hat{s}_{t+1} = \hat{f}(\hat{s}_t, a_t)$, $\hat{s}_0 = s_0$. Hence, for $E(\hat{s}_\tau, a_\tau)$ to imply something about $P(s_\tau, a_\tau)$ one needs additional smoothness assumptions (Lipschitzness of $E$ and $f$, and data-based error in $\hat{f}$).  We instantiate these assumptions and explore corresponding attainable guarantees the density model constrained MPC in Appendix \ref{subsec:ldm_barrier_proofs}. This analysis is in contrast to the LDM approach which is purely statistical, requiring fewer assumptions. 

To conclude, our results bound the error from using an \textit{imperfect} LDM as an MPC constraint, showing that it will more effectively constrain short-horizon MPC (with small values of $H$) than a na\"{i}ve density constraint.

%\textcolor{red}{$gamma=1.0$ works in practice, and can make easy direct comparison therefore the LDM provides a significantly more reliable constraint than the density model alone. and then in practice we expect K >> H.} \textcolor{red}{even if not, can take $\gamma$ big so}

%To conclude, in this section we analyzed the approximation error of learned LDMs (Alg.~\ref{ldm_alg_theoretical}) and gave bounds on the probability guarantee degradation when using an \textit{imperfect} LDM as an MPC constraint. We showed that despite errors, a learned LDM is still able to give long-horizon guarantees, avoiding the failure mode of short/medium-horizon MPC, or correspondingly the computational demands of very long horizon MPC, when using a na\"{i}ve density model barrier. This theoretical benefit of the LDM is confirmed experimentally in the next section.

\section{Experiments}\label{sec:experiments}
In this section, we present an experimental evaluation of our method. Our experiments aim to compare LDMs with other techniques for avoiding distribution shift. Prior works in model-based RL that consider this problem generally utilize density models or other error estimation schemes, such as ensembles \citep{chua2018deep,kidambi2020morel,yu2020mopo}. To provide an apples-to-apples comparison, we will compare LDMs to a baseline that uses a density model as a constraint (\ref{density_constraint}), as well as one that uses the variance of an ensemble of dynamics models as a constraint. These baselines are broadly representative of ``greedy'' constraints that forbid actions that violate some ``local'' metric, as opposed to the LDM, which takes into account \emph{future} state probabilities. %A density model also provides a constraint that is agnostic to the form of the downstream controller, like the LDM, and unlike an ensemble which is specific to the particular model-based controller that is used. 
We also compare to a na\"{i}ve baseline that does not employ any constraint at all. In our experiments, a static dataset is used to train the density model, LDM, ensemble models, and the dynamics model used for downstream control. We perform MPC with the learned dynamics model, analogous to prior model-based RL techniques~\citep{chua2018deep}, and constrain the planning procedure with the LDM or the baseline constraint functions. Further details about our implementation are given in Appendix~\ref{appendix:sec_9}.

We conduct our evaluation on two RL benchmark environments, hopper and lunar lander \citep{openai}, and a medical application, SimGlucose \citep{simglucose}. The objectives for the hopper and lunar lander tasks are to control the agent to reach different target locations (x position for hopper, and landing pad for lunar lander), and the objective for SimGlucose is to maintain the patient's glucose level close to a target setpoint. To effectively solve these tasks, the learning-based policy must not only accurately model complex dynamics (e.g. to perform a delicate landing procedure without crashing for lunar lander), but also be flexible to different target goals (e.g. different target x positions for the hopper require different behavior, such as hopping forwards or backwards, so directly copying the actions in the data would not be effective). For more details on the datasets and tasks, see Appendix \ref{appendix:data}. 

We aim to answer the following questions: 1) how does the value of the threshold used for the MPC constraint influence the behavior of the policy? 2) how does the performance of a MPC controller with an LDM constraint compare to that of using a density model constraint, an ensemble constraint, and no constraint?

\begin{figure}
\centering
\vspace{-0.1in}
\includegraphics[width=0.5\textwidth, trim={0 0 0 0},clip]{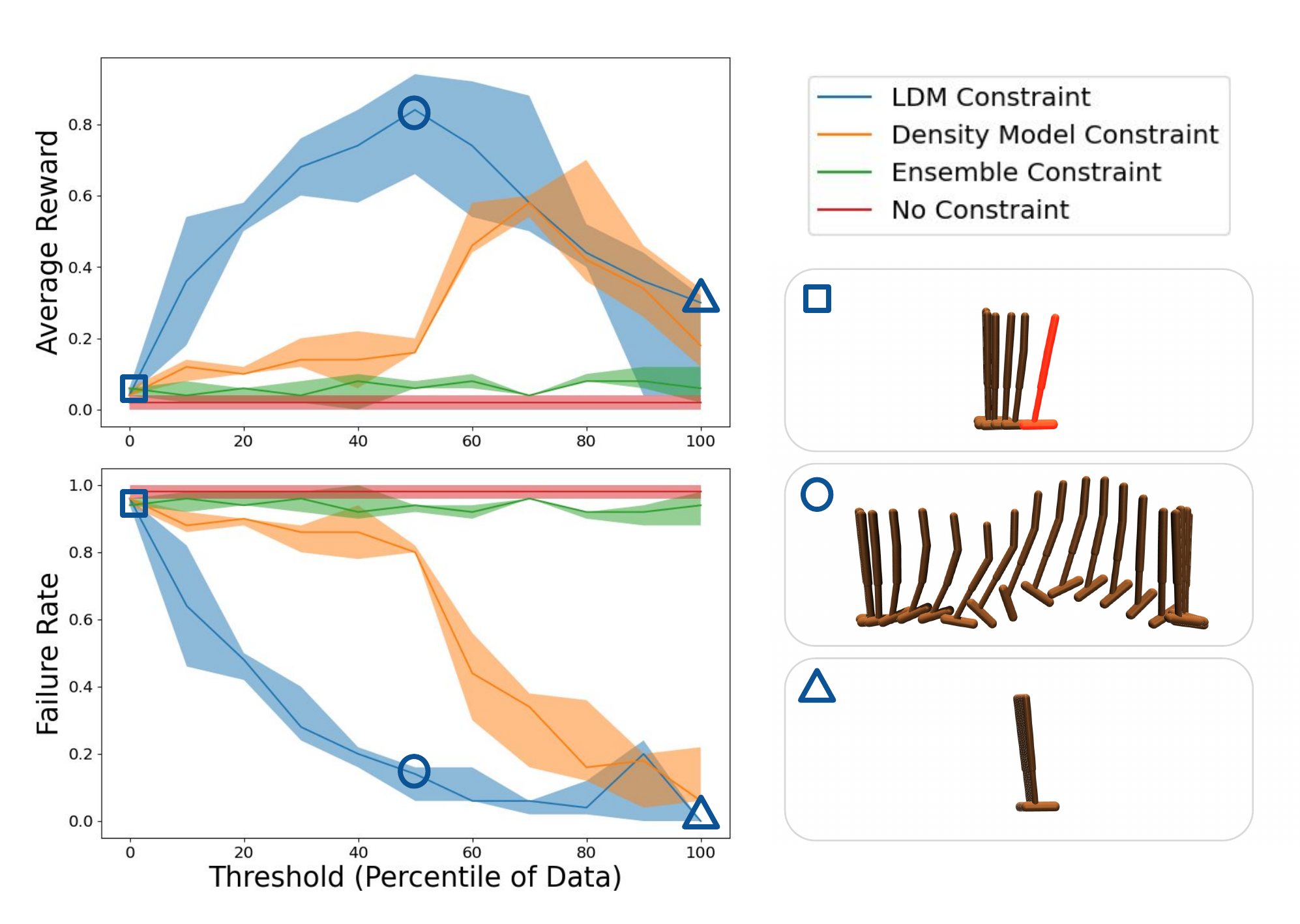}
\vspace{-0.3in}
\caption{\footnotesize Evaluation of average reward and failure rate over different threshold values for the hopper task (left), and example trajectories from an LDM-constrained policy with thresholds at 0, 50, and 100 (right). The x axis of the plots on the left represents the percent of the dataset whose value under the constraint function falls below the associated threshold $c$. More specifically, $x = \text{percentile}(\{W(\st^i, \at^i)\}_{i=1}^N, c)$, where $W$ is the constraint function. We use this representation of constraint values in order to make them comparable across different constraint functions. A hopper figure colored red represents termination due to failure. }
\vspace{-0.15in}
\label{thresholds_figure}
\end{figure}

To answer our first question on how the choice of the MPC constraint threshold influences the resulting policy, we aim to measure how the policy return and the degree to which the policy can keep the agent in-distribution changes as the threshold value changes. However, it is not completely clear how to objectively measure whether a generated trajectory is ``in-distribution.'' Our tasks involve a termination conditions that ends a trajectory when the agent enters a ``failure'' state (e.g., falling over for the hopper; becoming dangerously hypo/hyperglycemic for SimGlucose). Because the points in our dataset are concentrated in non-failing states, we use the rate of failure as a proxy for the fraction of trajectories that go out of distribution. 
% Note that our method does \emph{not} require the training dataset to be absent of failures. So long as there exist a portion of trajectories which do not terminate early, LDM-constrained MPC should encourage the agent to remain near those non-terminating trajectories, because the final state of a terminating trajectory will have low density, which a LDM-constrained policy is designed to avoid. 
Figure \ref{thresholds_figure} shows an example of the average return and failure rate across different thresholds for the hopper task (For the other tasks, see Appendix \ref{appendix:full_sweep}). We see that for low threshold values, the policies have low reward and high failure rate, as a consequence of excessive model exploitation (similar to the ``no constraint'' case). For high threshold values, the policies have low reward and low failure rate, because the over-conservative LDM constraint keeps the agent in-distribution but doesn’t leave enough flexibility for the agent to perform the desired task. For a concrete example, consider the figures of the hopper trajectories from the LDM constrained policy at different threshold values: the low threshold trajectory consists of the hopper falling over, the medium threshold trajectory has the hopper successfully hopping to the goal, and the high threshold trajectory consists of the hopper standing still, which effectively stays in-distribution but doesn’t accomplish the task. Thus, the threshold can be thought of as the user's knob for controlling the tradeoff between protecting against model error vs. flexibility for performing the desired task.

\begin{figure}
\centering
\vspace{-0.05in}
\includegraphics[width=0.48\textwidth, trim={0 0 0 0},clip]{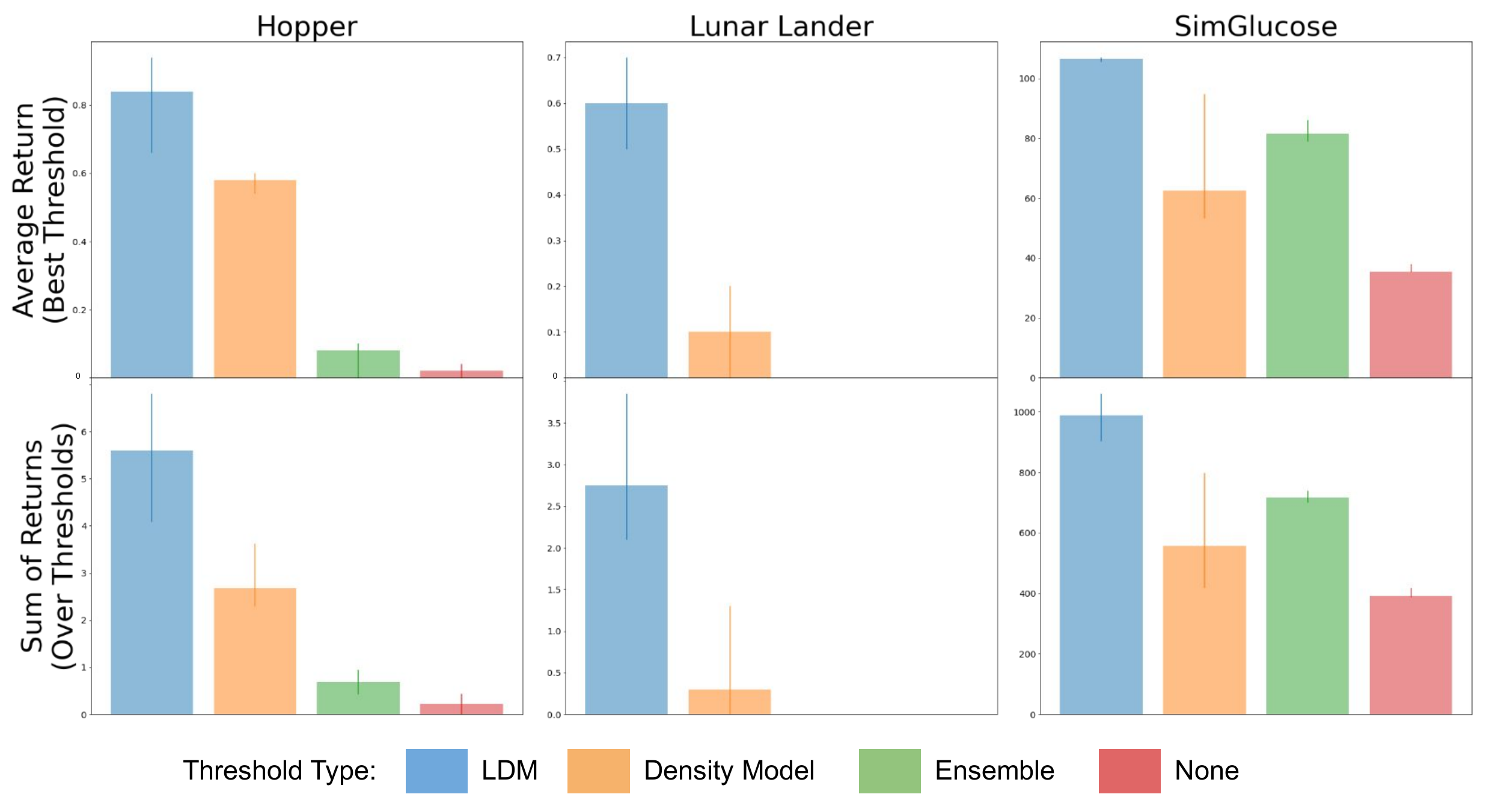}
\vspace{-0.2in}
\caption{\footnotesize Evaluation of the average reward attained by an MPC controller constrained with an LDM, a density model, an ensemble, and no constraint. The top row shows the best performance over thresholds for each method, and the bottom row shows the performance summed in aggregate over all thresholds. The bars show the median performance over random seeds, and the top and bottom of the error bars represent the 25/75 percentiles. We evaluated 5 random seeds for each task. }
\label{bar_figure}
\vspace{-0.15in}
\end{figure}

To answer our second question, we compare the performance of performing MPC with an LDM constraint, a density model constraint, an ensemble constraint, and no constraint for each task in Figure \ref{bar_figure}. The performance of the constrained MPC methods depends on the user-chosen threshold value for the constraints. For a fair comparison, we report the performance for the best threshold for each method in the top row. Furthermore, to assess each method's sensitivity to the choice of threshold, we also report the sum of the performances over a range of thresholds for each method (area underneath each curve in the average rewards plot of Fig. \ref{thresholds_figure}) in the bottom row. See Appendix \ref{appendix:full_sweep} for a full sweep of threshold values vs. performance for each task. We see that the unconstrained MPC policy was rarely successful in performing the task (red). In comparison, the density model (orange) and ensemble (green) constrained MPC methods yielded better performance by considering the data distribution while planning with a learned model. Finally, the LDM constrained MPC policy (blue) is able to most effectively perform the task for each tested environment, because it is able to reason about the data distribution in a dynamics aware fashion.

\vspace{0.15em}
\section{Discussion and Future Work}
We presented Lyapunov density models (LDMs), a tool that can ensure that an agent remains within the distribution of the training data. We provide a definition of the LDM, discuss its properties, and present a practical algorithm that learns LDMs from data. Furthermore, we provide a method of using an LDM in control, and present theoretical and empirical results showing its benefits to using density models. By making a formal connection between data distribution and control-invariance, we believe that our framework is a step towards developing learning-based control algorithms that resolve the issue of model unreliability due to distribution shifts. In this work, we focused on applying our framework in model-based RL; an exciting direction for future work would be to combine an LDM with other types of algorithms, such as model-free RL or imitation learning.

\newpage

\nocite{langley00}

\bibliography{example_paper}
\bibliographystyle{icml2022}

%%%%%%%%%%%%%%%%%%%%%%%%%%%%%%%%%%%%%%%%%%%%%%%%%%%%%%%%%%%%%%%%%%%%%%%%%%%%%%%
%%%%%%%%%%%%%%%%%%%%%%%%%%%%%%%%%%%%%%%%%%%%%%%%%%%%%%%%%%%%%%%%%%%%%%%%%%%%%%%
% APPENDIX
%%%%%%%%%%%%%%%%%%%%%%%%%%%%%%%%%%%%%%%%%%%%%%%%%%%%%%%%%%%%%%%%%%%%%%%%%%%%%%%
%%%%%%%%%%%%%%%%%%%%%%%%%%%%%%%%%%%%%%%%%%%%%%%%%%%%%%%%%%%%%%%%%%%%%%%%%%%%%%%
\newpage
\appendix
\onecolumn
\section{Appendix to Section~\ref{sec4_ldms}}
\subsection{Additional details of 2D Linear System Example \label{appendix:sec4_example}
}
% To illustrate the key concepts of the LDMs, we introduce an example of a discrete-time linear system, given by
% \begin{equation}
% \label{eq:linear}
%     s_{t+1} = F s_{t} + G a_t,
% \end{equation}
% where $F=e^{A \Delta t}$, $G=A^{-1}(e^{A \Delta t}-I)B$, with
% $A = \begin{bmatrix} \beta & \omega \\ -\omega & \beta \end{bmatrix}, B = \begin{bmatrix} 0 \\ 1 \end{bmatrix}$,
% % \vspace{-0.5em}
% % \small
% % \begin{equation*}
% %     A = \begin{bmatrix} \beta & \omega \\ -\omega & \beta \end{bmatrix}, B = \begin{bmatrix} 0 \\ 1 \end{bmatrix},
% % \end{equation*}
% % \normalsize
% %
% % \vspace{-1em}
% $s:=\begin{bmatrix}x_1 \\ x_2\end{bmatrix}$ and $\beta, \omega > 0$. %, $\Delta t = 2\pi / 12$, $\omega=1$, and $\beta=0.15$.
% Intuitively, when no action is applied, for each timestep the system \eqref{eq:linear} undergoes a rotation of $\omega \Delta t$ in the clockwise direction and a scaling of $e^{\beta \Delta t}$ in the radial direction. Since $\beta>0$, the origin is an unstable focus and the state will spirally diverge. The action can shift the state in the $x_2$--direction, which can be used to stabilize the system. Finally, for linear systems like \eqref{eq:linear}, an LQR-based stabilizing controller can be designed by solving a discrete-time Riccati equation, which results in a feedback policy $\pi_{LQR}(s)=-K s$. The behavior of the system is illustrated in Fig. \ref{fig_linear_appendix1}.

\begin{figure}[H]
\centering
\includegraphics[width=0.5\columnwidth, trim={0 0 0 0},clip]{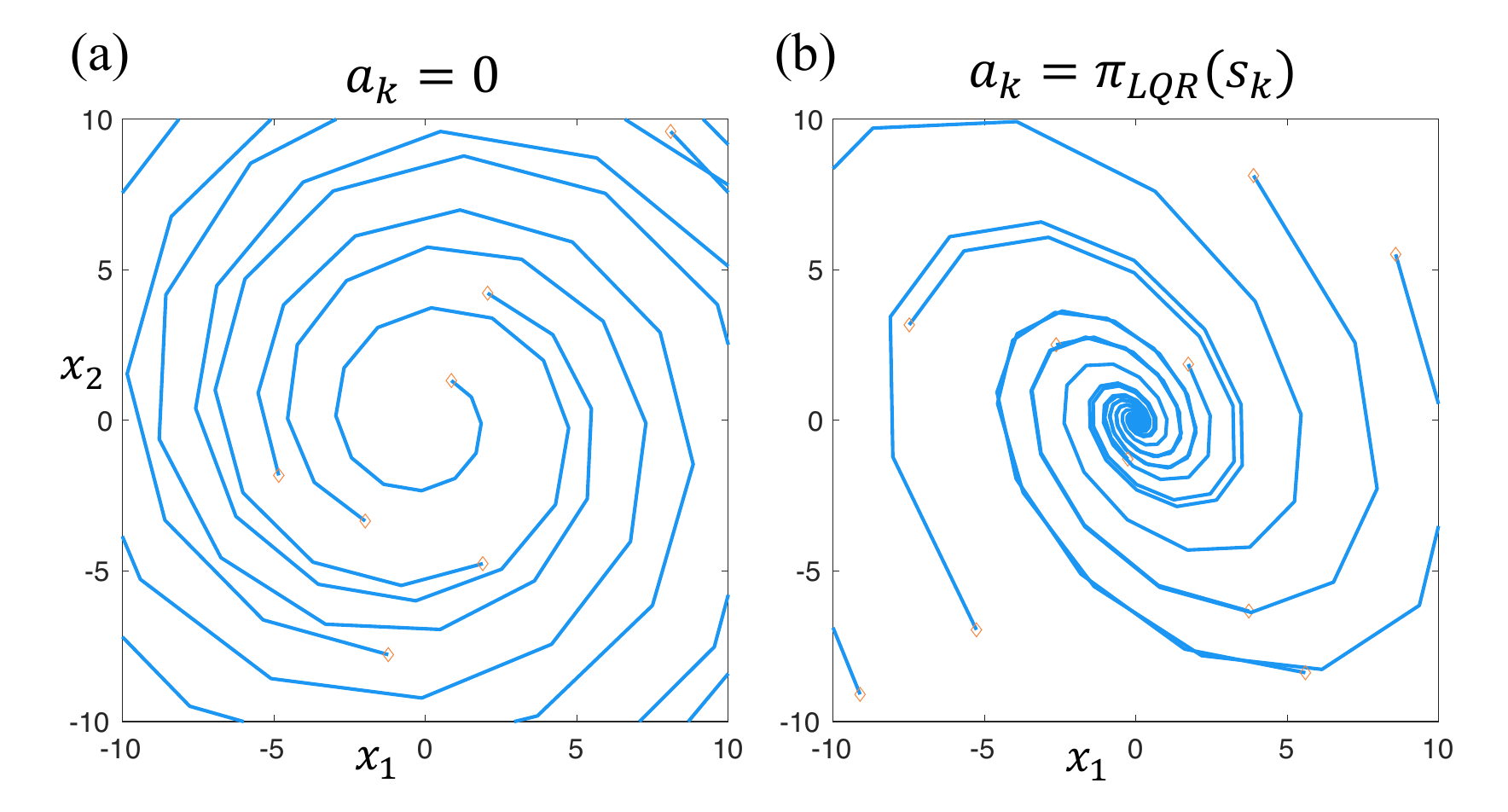}
\vspace{-2em}
\caption{Phase portraits of the linear system \eqref{eq:linear} under \textbf{(a)} no control ($a_k=0$) and \textbf{(b)} the LQR policy ($a_k=\pi_{LQR}(s_k)$), respectively. The uncontrolled dynamics of this system is unstable, therefore, there is no invariant region except for the origin. On the other hand, the LQR policy can stabilize the system to the origin, resulting in the trajectories to stay in a bounded region. Thus, the data distribution centered at this policy (case (b) in Fig. \ref{fig_linear_ldm}) will by its nature, have large state-action control invariant region.}
\label{fig_linear_appendix1}
\vspace{-1em}
\end{figure}

In Figure \ref{fig_linear_ldm} and \ref{fig_linear_appendix2}, level sets of the maximal LDMs $G(s, a)$ (on the right) that can be obtained for various distributions of the data density $P(s, a)$ (on the left) are visualized. The maximal LDMs \eqref{mdp-formulation} are computed on the grid of $(s, a)$ by applying the exact LDM backup \eqref{eq:bellman_op_density} described in Section \ref{sec:learning_ldms} with $\gamma=1$. The computation is done on the state-action grid of $[-10, 10]\!\times\![-10, 10]\!\times\![-5,5]$ with the size $(201,201,101)$, and it takes about a minute to obtain the converged $G(s,a)$ on a standard laptop. After obtaining the LDMs, from Theorem \ref{thm:ldm_base_guarantee}, the maximal state-control invariant sets which can maintain density levels $ P(\st, \at) \geq c$ are obtained by taking $\ldmlevelset := \{(s, a): \ldm(s, a) \leq -\log(c)\}$.

In addition to the two cases described in the main text, in Figure \ref{fig_linear_appendix2}, a toric data distribution is considered ((a) left). Note that the natural dynamics of \eqref{eq:linear} induces symmetric rotations. The invariant sets verified from the computed LDM ((a) right) reveal that indeed it is possible to maintain the system to stay in a torus with an appropriate density level. Also, we can obtain an policy for staying in this set by taking $\pi_{opt}(s) := \arg \min_{a\in\mathcal{A}}G(s, a)$. The resulting trajectory under this policy is visualized in (b).

\begin{figure}[H]
\includegraphics[width=\columnwidth, trim={0 0 0 0},clip]{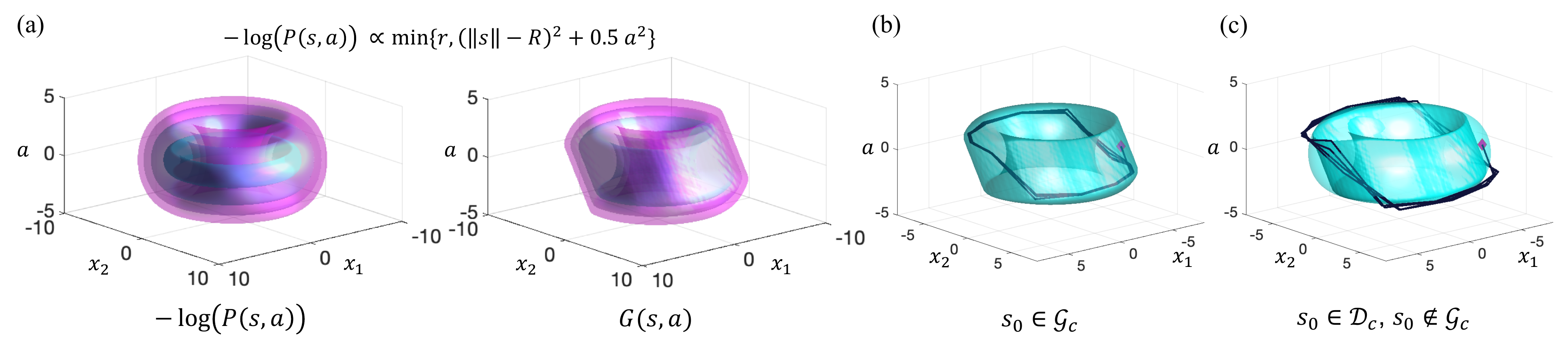}
\vspace{-2em}
\caption{The level sets of a toric data distribution for the linear system \eqref{eq:linear} are visualized in \textbf{(a)} left, and the level sets of the resulting maximal LDM are visualized in \textbf{(a)} right. In \textbf{(b)}, the smallest level set of $G(s,a)$ among the three sets in (a) are visualized in cyan again. The trajectory under the policy $\pi_{opt}(s) = \arg \min_{a\in\mathcal{A}}G(s, a)$, starting at an initial state $s_0$ (indicated by the diamond marker) that is contained in $\ldmlevelset = \{(s, a): \ldm(s, a) \leq -\log(c)\}$ is visualized in black. This shows that the obtained set $\ldmlevelset$ is indeed state-action control invariant which can be achieved by the policy $\pi_{opt}(s)$. In contrast, in \textbf{(c)}, an initial state that is not contained in $\ldmlevelset$ but still contained in $\mathcal{D}_c = \{(s, a) | P(s, a) \geq c\}$ (the bigger level set in cyan) is selected and the same optimal policy is tested out. Although the initial state is in-distribution, it is inevitable that the trajectory escapes $\mathcal{D}_c$.
}
\label{fig_linear_appendix2}
\vspace{-1em}
\end{figure}

\subsection{Proofs of Theorems \ref{thm:ldm_base_guarantee}, Proposition \ref{thm:max_ldm}}\label{appendix:sec4_proofs}
\begin{proof}[Proof of Theorem \ref{thm:ldm_base_guarantee}]
Suppose $(\st, \at) \in \mathcal{G}_p$. By Condition 1 in Definition \ref{ldm-definition}, we know that $\exists \atp \in \mathcal{A}$ such that $G(\st, \at) \geq G(f(\st, \at), \atp)$. Furthermore, since $\mathcal{G}_p$ is the sublevel-set of $G(\st, \at)$ at $-\log(p)$, we have $-\log(p) \geq G(\st, \at) \geq G(f(\st, \at), \atp) = G(\stp, \atp)$, so $(\stp, \atp) \in \mathcal{G}_p$. Now suppose $(\mathbf{s}_0, \mathbf{a}_0) \in \mathcal{G}_p$. By proof by induction, we know that $\exists \{\at\}_{t=1}^\infty$ such that $(\st, \at)\in \mathcal{G}_p \quad \forall t \geq 0$. Thus, $\mathcal{G}_p$ is a generalized control-invariant set. Finally, by Condition 2 in Definition \ref{ldm-definition}, we know that $G(\st, \at) \geq -\log(P(\st, \at))$, so $P(\st, \at) \geq p \quad \forall (\st, \at)\in \mathcal{G}_p$.
\end{proof}
\begin{proof}[Proof of Proposition \ref{thm:max_ldm}]
First, we show that $G(\state, \action)$, as defined by \ref{mdp-formulation}, is a Lyapunov density model. We prove each condition of an LDM individually:

\begin{itemize}
    % \item Continuity: $-\log(P(\st, \at))$ is continuous, and $G(\state_0, \action_0) = \min_{\{\at\}_{t=1}^\infty} \max_{t \geq 0} -\log(P(\st, \at))$ is the min and max of continuous functions, which is continuous.
    \item Condition 1: $\forall (\st, \at)$, $G(\st, \at) = \max\{-\log(P(\st, \at)), \min_{\atp}G(f(\st, \at), \atp)\}$ $\geq \min_{\atp}G(f(\st, \at), \atp)$. Thus, $\exists \atp$ such that $G(\st, \at) \geq G(f(\st, \at), \atp)$.
    \item Condition 2: $\forall (\st, \at)$, $G(\st, \at) = \max\{-\log(P(\st, \at)), \min_{\atp}G(f(\st, \at), \atp)\}$ $\geq -\log(P(\st, \at))$
\end{itemize}

Thus, by Theorem \ref{main-theorem}, each sub-level set of $G(\state, \action)$, $\mathcal{G}_p = \{(\st, \at) | G(\st, \at) \leq -\log(p)\}$, is a generalized control-invariant set such that $\forall (\st, \at) \in \mathcal{G}_p$, $P(\st, \at) \geq p$. To show that $G(\st, \at)$ is the \emph{maximal} LDM, we will next show that each $\mathcal{G}_p$ is the \emph{largest} generalized control invariant set contained inside $\mathcal{D}_p = \{(\st, \at) | P(\st, \at) \geq p\}$.

Suppose there exists a generalized control-invariant set $\mathcal{S}_p$ such that $\mathcal{G}_p \subset \mathcal{S}_p \subseteq \mathcal{D}_p$. Let $(\state_0, \action_0)$ be a point in $\mathcal{S}_p$ but not in $\mathcal{G}_p$. Because $\mathcal{S}_p$ is a generalized control-invariant set inside $\mathcal{D}_p$, for initial condition $(\state_0, \action_0)$, $\exists \{\at\}_{t=1}^\infty$ such that $\max_{t \geq 0} -\log(P(\st, \at)) \leq -\log(p)$. Because $(\state_0, \action_0)$ is not in $\mathcal{G}_p$, $G(\state_0, \action_0) > -\log(p)$. This is a contradiction, because $G(\st, \at)$ is by definition $\min_{\{\at\}_{t=1}^\infty} \max_{t \geq 0} -\log(P(\st, \at))$, so there cannot exist $\{\at\}_{t=1}^\infty$ such that $\max_{t \geq 0} -\log(P(\st, \at)) \leq -\log(p) < G(\state_0, \action_0)$. Thus, $\mathcal{G}_d$ must be the largest generalized control invariant set. 
\end{proof}
\subsection{LDM Connections to Density models and CLFs}
\label{appdx_connections_to_density_models}
There are two intuitive ways to view an LDM: 1) a Lyapunov function that ``stabilizes" towards a distribution rather than a single equilibrium point, or 2) a density model of the lowest data density over an agent's trajectory.LDMs are in fact a generalization of both Lyapunov functions and density models. For the special case of a static system $\stp = f(\st, \at) = \st$, $-\log(P(\st, \at))$ is a valid LDM for the data distribution $P(\st, \at)$. Note that this is just an example, not a necessary condition for the LDM to coincide with the density model.

On the other hand, an LDM can be viewed as a somewhat more expressive version of a Lyapunov function, because it is a function of states and actions, rather than just states, and is tied to a distribution rather than an equilibrium point. However, if the data distribution is peaked at an equilibrium point, then we can use the LDM to recover a CLF that stabilizes the system around that equilibrium point.

\begin{lemma}
Let $\mathbf{s}_e$ be an equilibrium point of the system, and $\mathbf{a}_e$ be the associated steady state action. If an LDM $\ldm(\st, \at)$ is radially unbounded and has a unique minimizer at $(\mathbf{s}_e, \mathbf{a}_e)$, then \mbox{$\clf(\st) = \min_{\at\in\mathcal{A}} \ldm(\st, \at) - \ldm(\mathbf{s}_e, \mathbf{a}_e)$} is a control Lyapunov function of the system.
\label{property2}
\end{lemma}

\textbf{Proof of Lemma \ref{property2}.} We prove each condition of a control Lyapunov function individually:
\begin{itemize}
    \item Continuity: Since $G(\st, \at)$ is continuous, $W(\st) = \min_{\at \in \mathcal{A}} G(\st, \at) - G(\seq, \aeq)$ is also continuous by Berge’s Maximum theorem. 
    \item Radial Unboundedness: Let $\at^* = \argmin_{\at \in \mathcal{A}} G(\st, \at)$. 
    $||\st|| \to \infty$ 
    $\implies ||[\st, \at^*]|| \to \infty$
    $\implies G(\st, \at^*) \to \infty$ $\implies W(\st) = \min_{\at \in \mathcal{A}} G(\st, \at) - G(\seq, \aeq) \to \infty $
    \item Condition 1: From Condition 1 of Definition \ref{ldm-definition}, we know that $\forall (\st, \at)$, $\exists \atp$ such that \mbox{$G(\st, \at) \geq G(f(\st, \at), \atp)$}. Let $\at^* = \argmin_{\action \in \mathcal{A}} G(\st, \at)$, $\exists  \atp$ such that $G(\st, \at^*) - G(\seq, \aeq) \geq G(f(\st, \at^*), \atp) - G(\seq, \aeq)$. Thus, $W(\st) = \min_{\at \in \mathcal{A}} G(\st, \at) \geq \min_{\atp \in \mathcal{A}} G(f(\st, \at^*), \atp) - G(\seq, \aeq) = W(f(\st, \at^*))$. Thus, $\forall \state \in \mathcal{S}$, $\exists \action \in \mathcal{A}$ such that $W(\state) \geq W(f(\state, \action))$.
    \item Condition 2: Since $(\seq, \aeq)$ is the unique minimizer of $G(\state, \action)$, $\forall (\state, \action)\neq (\seq, \aeq)$, $G(\state, \action) > G(\seq, \aeq)$. Thus, $\forall \state \neq \seq$, $\min_{\action \in \mathcal{A}} G(\state, \action) > G(\seq, \aeq)$, so $W(\state) = \min_{\action \in \mathcal{A}} G(\state, \action) - G(\seq, \aeq)$ $> 0$.
    \item Condition 3: $W(\seq) = \min_{\at \in \mathcal{A}} G(\seq, \at) - G(\seq, \aeq) =  G(\seq, \aeq) - G(\seq, \aeq) =0$.
\end{itemize}
%%SL.1.17: Here and also in Sec 4.2, be sure to mention where the reader can go to find the proofs for these theorems

%%SL.1.17: Generally, my sense is that Sec 4.1, 4.2, and 4.3 are a bit too drawn-out in their current form. Consider if maybe it's possible to shorten them somehow. Maybe some parts really belong in the theory/analysis section? Or maybe some parts can be combined together and shortened? I'm concerned that otherwise, we might not have room to fit everything that we need in the paper.

\subsection{Convergence of the value iteration algorithm for the maximal LDM.}\label{appendix:convergence}

\begin{lemma}
\label{lemma:bellman}
Define $G'_T(s_0 , a_0) := \min_{\{a_t\}_{t=1}^{T} \max_{t\in[0, T]} E(s_t,a_t)}$ with $E(s, a):=-\log P(s, a)$. Then
\begin{equation}
    G'_{T+1}(s_0 , a_0) = \max \left\{ E(s_0, a_0), \min_{a_1\in\mathcal{A}} G'_{T}(s_1, a_1) \right\},
\end{equation}
where $s_1 = f(s_0, a_0)$.
\end{lemma}
\begin{proof}
To simplify the notation, we use $a_{m:n}$ to indicate the sequence $\{a_t\}_{t=m}^{n}$.
\begin{align*}
    G'_{T+1}(s_0 , a_0) & = \min_{a_{1:T+1}} \max_{t\in[0, T+1]} E(s_t, a_t) = \min_{a_{1:T+1}} \max \left\{ E(s_0, a_0), \max_{t\in[1, T+1]} E(s_t, a_t) \right\} \\
    & = \max \left\{ E(s_0, a_0),  \min_{a_1 \in \mathcal{A}} \left\{ \min_{a_{2:T+1}} \max_{t\in[1, T+1]} E(s_t, a_t) \right\} \right\} = \max \left\{ E(s_0, a_0), \min_{a_1\in\mathcal{A}} G'_{T}(s_1, a_1) \right\}.
\end{align*}
\end{proof}

Note that Lemma \ref{lemma:bellman} implies that the finite-horizon version of the maximal LDM satisfies the Bellman principle associated with the LDM backup operator $\mathcal{T}$ \eqref{eq:bellman_op_density} with $\gamma=1$.

\begin{theorem}
\label{th:convergence}
 The value iteration algorithm with $G_{k+1} \leftarrow \mathcal{T}G_{k} $, with $\mathcal{T}$ defined in \eqref{eq:bellman_op_density} under $\gamma=1$ and with initialization $G_0(s, a)\!=\!-\log\left(P(s,a)\right)$, results in 1) $G_k = G'_k$ for $\forall \; k \ge 0$, and 2) for any $(s,a)\in \mathcal{S}\times\mathcal{A}$ such that the maximal LDM $G(s, a)$ defined in \eqref{mdp-formulation} is finite, $\lim_{k\rightarrow \infty} G_k(s, a) = G(s, a)$.
\end{theorem}

\begin{proof}
1) is a direct corollary of Lemma \ref{lemma:bellman}. Then, by the definition of $G'_{k}$,
\begin{align*}
    G_{k+1}(s, a) = \min_{a_{1:k+1}} \max_{t\in[0, k+1]} E(s_t,a_t) & = \min_{a_{1:k+1}} \max \left\{ \max_{t\in[0, k]} E(s_t,a_t), E(s_{k+1}, a_{k+1}) \right\} \\
    & \ge \min_{a_{1:k+1}} \max_{t\in[0, k]} E(s_t,a_t) = G_k(s, a).
\end{align*}
In a similar way, one can easily proof that $G(s, a) \ge \min_{a_{1:\infty}} \max_{t\in[0, k]} E(s_t,a_t) = G_k(s, a)$ for $\forall \; k \ge 0$. Therefore, for all $(s,a)\in \mathcal{S}\times\mathcal{A}$ such that $G(s, a)$ is finite, $\{G_k(s,a)\}_{k=0}^{\infty}$ is a non-decreasing sequence bounded above. By monotone convergence theorem, this sequence converges to its supremum which is by definition, the maximal LDM $G(s, a)$. 
\end{proof}

\section{Appendix to Section~\ref{sec:ldm_control}}\label{appendix:ldm_control}

\begin{proof}[Proof of Proposition \ref{ex:baseline_failure}] Let $K$ be such that $1/K \leq 2(H+1) \epsilon$. We construct $\mathcal{D}$ by its generating law as follows:
\begin{align*}
P_{\mathcal{D}}(s, -1) = \begin{cases} \frac{1}{2(H+1)} \text{ if } s\in \{-(H-1), \ldots, 0\} \\
0 \;\;\;\;\;\;\;\;\; \text{ o.w.}
\end{cases} & P_{\mathcal{D}}(s, 1) = \begin{cases} \frac{1}{2(H+1)} \text{ if } \in \{0, \ldots, H-1\} \\
0 \;\;\;\;\;\;\;\;\; \text{ o.w.}
\end{cases} \\
P_{\mathcal{D}}(-H, 0) = \frac{1}{2(H+1)} \;\;\;\;\;\;\;\;\;\;\;\;\;\;& 
P_{\mathcal{D}}(H, k) = \begin{cases} \frac{1}{2(H+1)\mathbf{K}} \text{ for } k \in \{0, \ldots, (K-1)\} \\
0 \;\;\;\;\;\;\;\;\; \text{ o.w.} 
\end{cases} 
\end{align*}
and $P_{\mathcal{D}}(s,a) = 0$ wherever not defined above. Note that the above probabilities sum to $1$ and that the maximal $p$ for which the MPC problem with a na\"{i}ve distribution constraint is feasible is $p^\star = \frac{1}{2(H+1)}$. Because our reward is $a$ and because taking $a_{1:H} = 1$ satisfies the constraint at the very first timestep $t=1$, we will take $a_1 = 1$. From hereon out, the dataset contains only `1' actions on states $\overline{1, H-1}$ so we are doomed to always take $a_t = 1$ and hence reach $s_{H+1} = H$ where by construction $P_{\mathcal{D}}(s_{H+1}, a) < \epsilon$ for any action $a$.

This construction corresponds to data collected from a game where the agent can go right or left or stay still with $-H$ being the only terminal state and that the player deterministically stops the collection after its $H+1^{th}$ action. This time, if it took $H$ actions and didn't finish, it takes one random action to see what happens then ends the data collection. The uncertainty is therefore caused by the unstable behavior of the player at state $-H$. 

Finally, it remains to show that the LDM has the promised guarantee on this construction. First note that taking $$a_t = \begin{cases} -1 \text{ if } s < H \\
0 \text{ o.w.}\end{cases}$$guarantees we always have $P_{\mathcal{D}}(s_t, a_t) \geq \frac{1}{2(H+1)}$ arbitrarily far in the future. Using any non-trivial LDM, and in particular the maximal LDM, as a constraint will guarantee that we satisfy $P(s_t, a_t)$ due to Theorem~\ref{thm:ldm_base_guarantee}. Hence, via the LDM constraint, even if the data was collected from short trajectories, we can give infinite-horizon guarantees.
\end{proof}

\section{Appendix to Section~\ref{sec:theory}}\label{app:theory}
We first prove Proposition \ref{prop:ldm_fqi} as stated in the main text in \ref{subsec:ldm_fqi_proof}, then derive bounds on $\epsilon_{\mathrm{ls}}$ with an associated corollary that contains all the dependencies in \ref{subsec:eps_ls_bounds}.
\subsection{Proof of Proposition \ref{prop:ldm_fqi}}\label{subsec:ldm_fqi_proof}

\begin{proof}[Proof of Proposition \ref{prop:ldm_fqi}]\label{proof:fqi} Consider arbitrary $G_1, G_2$ and their optimal safety policies. If $G_1(f(s,a), \pi_1(f(s,a))) \leq G_2(f(s,a), \pi_2(f(s,a))$, we have that:
\begin{align*}
\mathcal{T} G_2(s,a) - \mathcal{T} G_1(s,a) &= \max\{E(s,a), \gamma G_2(f(s,a), \pi_2(f(s,a)))\} - \max\{E(s,a), \gamma G_1(f(s,a), \pi_1(f(s,a)))\} \\
&\leq \max\{E(s,a), \gamma G_2(f(s,a), \pi_1(f(s,a)))\} - \max\{E(s,a), \gamma G_1(f(s,a), \pi_1(f(s,a)))\} \\
&= \begin{cases}
0 \text{ if } \gamma G_2(f(s,a), \pi_1(f(s,a))) \leq E(s,a) \\
\gamma G_2(f(s,a), \pi_1(f(s,a))) - \max\{E(s,a), \gamma G_1(f(s,a), \pi_1(f(s,a)))\} \text{ o.w.}\end{cases}\\
&\leq \gamma (G_2(f(s,a), \pi_1(f(s,a))) - G_1(f(s,a), \pi_1(f(s,a))))
\end{align*}
otherwise (if $G_1(f(s,a), \pi_1(f(s,a))) > G_2(f(s,a), \pi_2(f(s,a))$), analogously (by symmetry) we have
\begin{align*}
\mathcal{T} G_1(s,a) - \mathcal{T} G_2(s,a) &\leq  \gamma (G_1(f(s,a), \pi_2(f(s,a))) - G_2(f(s,a), \pi_2(f(s,a))))
\end{align*}
so letting $\pi_{12}(s) \doteq \begin{cases}\pi_1(s) \text{ if } G_1(s, \pi_1(s)) \leq G_2(s, \pi_2(s)) \\
\pi_2(s) \text{ o.w.}\end{cases}$ we have 
$$|\mathcal{T} G_1(s,a) - \mathcal{T} G_2(s,a)| \leq \gamma|G_1(f(s,a), \pi_{12}(f(s,a))) - G_2(f(s,a), \pi_{12}(f(s,a)))|$$
We can apply this recursively to obtain that for any $s, a$ and $t\geq 1$, there exists a sequence of actions $a_0=a$, $a_{1:t}$ and corresponding states $s_{0} = s$, $s_{k+1} = f(s_k, a_k)$ for $k=\overline{0,t-1}$ such that
$$|\mathcal{T}^{(t)} G_1(s,a) - \mathcal{T}^{(t)} G_2(s,a)|\leq \gamma^t |G_1(s_t, a_t) - G_2(s_t, a_t)|$$
and using the recoverability factor from Definition~\ref{recoverability_def}, we can bound (note that $(s,a)=(s_0,a_0)$):
\begin{align*}\|\mathcal{T}^{(t)} G_1(s,a) - \mathcal{T}^{(t)} G_2(s,a)\|_P &\leq \gamma^t \sum_{s_0, a_0} |G_1(s_t, a_t) - G_2(s_t, a_t)| \cdot P(s_0, a_0) \\
&\leq R \gamma^t \sum_{s_t, a_t} |G_1(s_t, a_t) - G_2(s_t, a_t)| \cdot P(s_t, a_t)\\
&= R \gamma^t \|G_1 - G_2\|_P
\end{align*}
or, for continuous state-action spaces,
\begin{align*}\|\mathcal{T}^{(t)} G_1(s,a) - \mathcal{T}^{(t)} G_2(s,a)\|_P &\leq \gamma^t \int_{s_0, a_0} |G_1(s_t, a_t) - G_2(s_t, a_t)| \cdot P(s_0, a_0) \\
&\leq R \gamma^t \int_{s_t, a_t} |G_1(s_t, a_t) - G_2(s_t, a_t)| \cdot P(s_t, a_t)\\
&= R \gamma^t \|G_1 - G_2\|_P
\end{align*}
Now since $\mathcal{T}\hat{G}^\star_\gamma = \hat{G}^\star$, by triangle inequality we have:
\begin{align*}
\|\hat{G}_{K} - \hat{G}^\star\|_{P} &\leq \|\mathcal{T} \hat{G}_{K-1} - \mathcal{T} \hat{G}^\star\|_{P} + \|\hat{G}_{K} - \mathcal{T} \hat{G}_{K-1}\|_{P}\\
&\leq \|\mathcal{T}^{(2)} \hat{G}_{K-2} - \mathcal{T}^{(2)} \hat{G}^\star\|_{P} + \|\mathcal{T} \hat{G}_{K-1} - \mathcal{T}^{(2)} \hat{G}_{K-2}\|_{P} + \epsilon_{\mathrm{ls}} \\
&\leq \|\mathcal{T}^{(2)} \hat{G}_{K-2} - \mathcal{T}^{(2)} \hat{G}^\star\|_{P} + R \gamma \epsilon_{\mathrm{ls}} + \epsilon_{\mathrm{ls}} \\
&\leq \ldots \\
&\leq \|\mathcal{T}^{(K)} \hat{G}_{0} - \hat{G}^\star\|_{P} + R(\gamma^{K-1} + \ldots + 1) \cdot \epsilon_{\mathrm{ls}} \\
&\leq \frac{R(1-\gamma^K)}{1-\gamma} \cdot \epsilon_{\mathrm{ls}} + \min\{R \gamma^K \|E-\hat{G}^\star\|_P, \gamma^K \|E-\hat{G}^\star\|_\infty\}
\end{align*}
which is clearly less than the bound in the Proposition statement. Note that there may be cases where $R \gamma^K \|E-\hat{G}^\star\|_P < \gamma^K \|E-\hat{G}^\star\|_\infty$, i.e. if the maximal deviation is very large but the $P$-weighted deviation is not. We chose the stated formulation to disentangle the dependency on $R$ for presentability purposes. 

As remarked in footnote, if we have $\gamma=1$, we would instead get:
$$\|\hat{G}_{K} - \hat{G}^\star\|_{P} \leq R\cdot K \cdot \epsilon_{\mathrm{ls}} + \sup_{P^{(K)}}\|E-\hat{G}^\star\|_{P^{(K)}}$$
where the supremum is taken over all probability distributions that could be reached by acting according to any arbitrarily complicated, possibly non-stationary policy for $K$ steps (applied to each state-action pair). This bound could be useful if regardless of how we act $P^{(K)}$ concentrates to a region where $E$ and $\hat{G}^\star$ coincide. It is envisonable that this could happen for stable systems for example.

Finally, Remark \ref{rmk:fqi_conv_2} follows by considring only $t=1$, using the definition of one-step recoverability and recursively obtaining:
\begin{align*}
\|\hat{G}_{K} - \hat{G}^\star\|_{P} &\leq \|\mathcal{T} \hat{G}_{K-1} - \mathcal{T} \hat{G}^\star\|_{P} + \|\hat{G}_{K} - \mathcal{T} \hat{G}_{K-1}\|_{P}\\
&\leq r\gamma \|\hat{G}_{K-1} - \hat{G}^\star\| + \epsilon_{\mathrm{ls}} \\
&\leq \ldots \\
&\leq \frac{(1-(r\gamma)^K)\cdot \epsilon_{\mathrm{ls}}}{1- r\gamma} + (r\gamma)^K\|E-\hat{G}^\star\|_P
\end{align*}
as stated.
\end{proof}

Depending on the tradeoff between $r$ and $R$, different approaches to choosing $\gamma$ may be better. For example, if $r \ll R$ or if $r$ is very close to $1$, then it would be worthwhile to select $\gamma = r^{-1}$ so that we can rely on the guarantees of Remark~\ref{rmk:fqi_conv_2} instead. On the other hand, if $r$ is not much smaller than $R$ it could be better to just try to optimize for best $\gamma$ closer to $1$ and rely on the main guarantee.

% We underline the last point made in the proof above regarding the ability to replace dependence $R$ with one on the one-step recoverability $r$ below:

% \begin{remark}[Error in terms of One-step Concentrability ]\label{rmk:fqi_conv_2}
% $\hat{G}_K$ also satisfies $\|\hat{G}_K - \hat{G}^\star\|_P \leq \frac{\epsilon_{\mathrm{ls}}}{1-r\gamma} + (r\gamma)^K \cdot \|E - \hat{G}^\star\|_{P}$ for $\gamma < r^{-1}$.
% \end{remark}

We also make a final comment the converge behavior we expect in scenarios where Remark~\ref{rmk:gamma_1} holds.

\begin{remark}[Convergence to undiscounted LDM] In this case if the dataset has an expected trajectory length $K_{\mathrm{avg}}$, then by Markov's inequality we can get that $K_{\mathrm{fin}} \leq K_{\mathrm{avg}} \cdot (\# \text{ traj})^{1/4}$ with probability at least $1- (\# \text{ traj})^{-1/4}$. As such, since $|D| \approx K_{\mathrm{avg}} \cdot (\# \text{ traj})$, we could expect that with high probability, $K\times \epsilon_{\mathrm{ls}}$ scales as $(\# \text{ traj})^{-1/4}$, hence ensuring converge to the undiscounted LDM.\end{remark}

\subsection{Bounding the Generalization Error $\epsilon_{\mathrm{ls}}$}\label{subsec:eps_ls_bounds}

We can state the following standard results based on some complexity measure of the function class $\mathcal{G}$. For example we can use the following standard result:

\begin{lemma}[Least Squares Generalization Bound]\label{lem:ls_generalization} Given a dataset $\mathcal{D} = \{x_i, y_i\}_{i=1}^n$ where $x_i\in \mathcal{X}$ and $x_i, y_i \sim \mu$, $y_i = f^\star(x_i) + \epsilon_i$, where $|y_i| \leq Y$ and $|\epsilon_i|\leq \sigma$ for all $i$, and $\{\epsilon_i\}$ are independent from each other. Given a function class $\mathcal{F}: \mathcal{X} \rightarrow [0, Y]$ such that $\min_{f\in\mathcal{F}}\mathbb{E}_{x\sim \mu}(f^\star(x) - f(x))^2 \leq \epsilon_{\mathrm{aprx}}$. If we denote $\hat{f} = \arg\min_{f\in\mathcal{F}}\sum_{i=1}^n (f(x_i)- y_i)^2$. With probability at least $1-\delta$, we have:
$$\mathbb{E}_{\mu}\left[(\hat{f}(x) - f^\star(x))^2\right] \leq \frac{144 Y^2 \ln{(|\mathcal{F}|/\delta)}}{n} + 8 \epsilon_{\mathrm{aprx}}$$
\end{lemma}

For a proof see Lemma $6$ from \citet{hsu2014random} or Lemma A.11. from \citet{rltheorybook}.

\begin{corollary}With probability at least $1-\delta$ we have
$$\epsilon_{\mathrm{ls}} \leq \frac{12 E_{\mathrm{max}}\sqrt{\ln{(|\mathcal{G}|^2 K/\delta)}}}{\sqrt{|D|}} + 3 \sqrt{\epsilon_{\mathrm{aprx}}}$$
where $\epsilon_{\mathrm{aprx}} \doteq \max\limits_{G\in\mathcal{G}} \min\limits_{G' \in \mathcal{G}} \mathbb{E}_P\left[\left(G'(s,a) - \mathcal{T}G(s,a)\right)^2\right]$ is the inherent Bellman error and $E_{\mathrm{max}} \doteq \max\limits_{(s,a) \text{ s.t. } P(s,a) > 0} E(s,a)$ is the maximal value our density model takes on regions that have $P(s,a) > 0$.
\end{corollary}

\subsection{Proofs for Section~\ref{sec:mpc_theory}}\label{subsec:ldm_barrier_proofs}

We will prove the guarantee under the `with high probability' version of the density model error assumption from the main text:
\begin{assumption}\label{assump:density}There exist $\epsilon_p  \geq 0, \delta_p \in [0,1]$ such that
$$\mathbb{P}_{(s, a)\sim_{\mathrm{unif.}} \mathcal{S}\times\mathcal{A}}\left[|E(s,a) + \log P(s,a)| > \epsilon_p \right] \leq \delta_p$$
i.e. at most a fraction $\delta_p$ of the state-action space has error bigger than $\epsilon_p$.
\end{assumption}

We have the following guarantee:
\begin{proposition}\label{generalized_degradation_bd}
For any probability mass function $P$, starting from any state $s_0$, taking action $a_0$ s.t. $\hat{G}(s_0, a_0) \leq -\log{c}$ guarantees that $\forall t$, there is a sequence $a_{1:t}$ for which with probability at least $1-2\delta_p$
$$\log{P(s_t, a_t)} \geq \gamma^{-t} \log{c} - \frac{\gamma^{-t} R\epsilon_{\mathrm{ls}}\exp{\epsilon_p}}{(1-\gamma)c} - \frac{\gamma^{K-t} \exp{\epsilon_p} \|E - \hat{G}^\star\|_{\infty}}{c} - \epsilon_p$$
if we use a discount $\gamma < 1$. Alternatively, if we use $\gamma = 1$, with probability at least $1-2\delta_p$, 
$$\log P(s_t, a_t) \geq \log{c} - \frac{R\cdot K \cdot \epsilon_{\mathrm{ls}}\, \exp{\epsilon_p}}{c} - \frac{\exp{\epsilon_p} \cdot \sup_{P^{(K)}}\|E-\hat{G}^\star\|_{P^{(K)}}}{c} - \epsilon_p$$
where $P^{(K)}$ is any probability distributions that can be reached by some sequence of actions (see Proof~\ref{proof:fqi} for more).
\end{proposition} 

Proposition~\ref{generalized_degradation_bd} clearly encompasses the first clause of Proposition~\ref{prop:ldm_degradation} by taking $\delta_p = 0$ and $K\rightarrow \infty$. The second clause from the main text also follows directly in the case that the scenario in Remark~\ref{rmk:gamma_1} occurs. As a minor side note, we can bound $\exp{\epsilon_p} \leq 1 + \epsilon_p + \epsilon_p^2$ for $\epsilon_p < 1.79$ in order to avoid having both logarithms and exponentials.

\begin{proof}[Proof of Proposition~\ref{generalized_degradation_bd}] First note that we can always replace the LDM with \mbox{$\hat{G}(s,a) = \max\{\hat{G}_K(s,a), E(s,a)\}$}, which never increases the errors in \ref{prop:ldm_fqi}/\ref{rmk:gamma_1}/\ref{rmk:fqi_conv_2} and can be seen as a refinement of avoidable fitting error.\footnote{This is a natural slight modification that we have deferred simply for presentation purposes.} This refinement ensures we have $\hat{G}(s_0, a_0) \leq -\log{c}$ implies that $E(s_0, a_0) \leq -\log{c}$. Due to Assumption~\ref{assump:density}, with probability at least $1-\delta_p$, $E(s_0, a_0) \geq -\log P(s_0, a_0) - \epsilon_p$ and hence by rearranging $P(s_0, a_0) \geq c e^{-\epsilon_p}$. Since for now we assume $P$ is a probability mass function, Proposition~\ref{prop:ldm_fqi} implies that \begin{align*}
\quad\quad &|\hat{G}(s_0, a_0) - \hat{G}^\star(s_0,a_0)| \leq \frac{R\epsilon_{\mathrm{ls}}e^{\epsilon_p}}{(1-\gamma)c } + \frac{\gamma^K e^{\epsilon_p} \|E - \hat{G}^\star\|_{\infty}}{c} \\    
\Rightarrow \quad & \hat{G}^\star(s_0, a_0) \leq \hat{G}(s_0, a_0) + \frac{R\epsilon_{\mathrm{ls}}e^{\epsilon_p}}{(1-\gamma)c } + \frac{\gamma^K e^{\epsilon_p} \|E - \hat{G}^\star\|_{\infty}}{c}
\end{align*}
by the definition of $\hat{G}^\star$ this implies that for any $t$ there is a sequence of actions $\{a_\tau\}_{\tau=1}^t$ for which starting from $s_0, a_0$ we can guarantee
$$\gamma^t E(s_t, a_t) \leq \hat{G}^\star(s_0, a_0)\leq - \log{c} + \frac{R\epsilon_{\mathrm{ls}}e^{\epsilon_p}}{(1-\gamma)c } + \frac{\gamma^K e^{\epsilon_p} \|E - \hat{G}^\star\|_{\infty}}{c}$$
Since the probability that Assumption~\ref{assump:density} breaks twice is bounded by $2\delta_p$ due to union bound, we have that with probability at least $1-2\delta_p$,
\begin{align*}
\quad \quad &\gamma^t (-\log{P(s_t, a_t)} -\epsilon_p) \leq - \log{c} + \frac{R\epsilon_{\mathrm{ls}}e^{\epsilon_p}}{(1-\gamma)c} + \frac{\gamma^K e^{\epsilon_p} \|E - \hat{G}^\star\|_{\infty}}{c} \\    
\Rightarrow \quad & \gamma^t (\log{P(s_t, a_t)} + \epsilon_p) \geq \log{c} - \frac{R\epsilon_{\mathrm{ls}}e^{\epsilon_p}}{(1-\gamma)c} - \frac{\gamma^K e^{\epsilon_p} \|E - \hat{G}^\star\|_{\infty}}{c} \\
\Rightarrow \quad & \log{P(s_t, a_t)} \geq \gamma^{-t} \log{c} - \frac{\gamma^{-t} R\epsilon_{\mathrm{ls}}e^{\epsilon_p}}{(1-\gamma)c} - \frac{\gamma^{K-t} e^{\epsilon_p} \|E - \hat{G}^\star\|_{\infty}}{c} - \epsilon_p
\end{align*}
The second statement follows analogously from the guarantee from Proof~\ref{proof:fqi} which implies
$$\hat{G}^\star(s_0, a_0) \leq \hat{G}(s_0, a_0) + \frac{R\cdot K \cdot \epsilon_{\mathrm{ls}}\, e^{\epsilon_p}}{c} + \frac{e^{\epsilon_p} \cdot \sup_{P^{(K)}}\|E-\hat{G}^\star\|_{P^{(K)}}}{c}$$
where $\hat{G}^\star$ is now undiscounted so we directly obtain
$$\log P(s_t, a_t) \geq \log{c} - \frac{R\cdot K \cdot \epsilon_{\mathrm{ls}}\, e^{\epsilon_p}}{c} - \frac{e^{\epsilon_p} \cdot \sup_{P^{(K)}}\|E-\hat{G}^\star\|_{P^{(K)}}}{c} - \epsilon_p$$
concluding the full statement of the bound.
\end{proof}

Now we plug the above in multiple times and do some manipulation to get the form presented in Corollary \ref{prop:reward_ldm_degradation}:

\begin{proof}[Proof of Corollary \ref{prop:reward_ldm_degradation}]
Using our assumption on the error scaling we have that:
\begin{align*}
|\hat{R}_T(\mathbf{s}, \mathbf{a}) - R_T(\mathbf{s}, \mathbf{a})| &= \sum_{t=0}^{T-1} \gamma^t\left(r(f(s_t, a_t), a_{t+1}) - r(\hat{f}(s_t, a_t), a_{t+1})\right) \\
&\leq \sum_{t=0}^{T-1} \frac{\gamma^t\cdot \epsilon_r}{\sqrt{P(s_t, a_t)}} \\
&= \sum_{t=0}^{T-1} \frac{\gamma^t }{\sqrt{\epsilon_r^{-2}P(s_t, a_t)}}
\end{align*}
For $\epsilon_r^{-2} P(s_t, a_t) \geq 0.08104$, we have that $\frac{1}{\sqrt{\epsilon_r^{-2} P(s_t, a_t)}} \leq 1 - \log{(\epsilon_r^{-2} P(s_t, a_t))} = 1 + 2\log{\epsilon_r)} - \log{ P(s_t, a_t)} $ which enables us to plug in the guarantee of Proposition~\ref{prop:ldm_degradation} multiplied by $\gamma^t$ iteratively. For now let's assume we picked $c$ such that $P(s_t, a_t) \geq \epsilon_r^2 \cdot 0.08104$ (note that for good enough model this is \underline{very} small) for $t=\overline{0,T-1}$ (we will derive the setting required at the end). In that case, we have:
\begin{align*}
|\hat{R}_T(\mathbf{s}, \mathbf{a}) - R_T(\mathbf{s}, \mathbf{a})| &= \sum_{t=0}^{T-1} \gamma^t \cdot (1 + 2\log{\epsilon_r} - \log P(s_t, a_t)) \\
&\leq  \sum_{t=0}^{T-1} \left[\gamma^t (1 + 2\log{\epsilon_r)}) -  \gamma^t \log P(s_t, a_t) \right]\\
&\leq \sum_{t=0}^{T-1} \left[ \gamma^t (1 + 2\log{\epsilon_r}) + \gamma^t \epsilon_P + \cdot \left( -\log{c} + \frac{R\epsilon_{\mathrm{ls}}\exp{\epsilon_P}}{c(1-\gamma)}\right)\right] \\
&= (1 + \epsilon_P + 2\log{\epsilon_r}) \cdot \frac{1 - \gamma^T}{1-\gamma} + T \cdot \left(\log{\frac{1}{c}} + \frac{R \epsilon_{\mathrm{ls}} \exp{\epsilon_P}}{c(1-\gamma)}\right)
\end{align*}
as desired. Finally we need to see what setting of $c$ ensures we can apply the $\frac{1}{\sqrt{x}} \leq 1 -\log{x}$ inequality throughout the trajectory. Using again our main guarantee it suffices to ensure
\begin{align*}
\log{c} - \frac{R \epsilon_{\mathrm{ls}}\exp{\epsilon_P}}{c(1-\gamma)}  &\geq \gamma^{T-1}(\epsilon_P + 2\log{\epsilon_r} + \log{0.08104})   
\end{align*}
Hence since $\log{x} \geq 1 -\frac{1}{x}$, it suffices to ensure
\begin{align*}
&1 - \frac{1}{c} - \frac{R \epsilon_{\mathrm{ls}}\exp{\epsilon_P}}{c(1-\gamma)}  \geq \gamma^{T-1}(\epsilon_P + 2\log{\epsilon_r} + \log{0.08104})   \\
\iff & 1 - \gamma^{T-1} (\epsilon_P + 2\log{\epsilon_r} + \log{0.08104}) \geq \frac{(1-\gamma) + R \epsilon_{\mathrm{ls}}\exp{\epsilon_P}}{c(1-\gamma)} \\
\iff & c \geq \frac{(1-\gamma) + R \epsilon_{\mathrm{ls}}\exp{\epsilon_P}}{(1 - \gamma^{T-1} (\epsilon_P + 2\log{\epsilon_r} + \log{0.08104})) (1-\gamma)} \label{full_lb_reward_theory}
\end{align*}
giving the tighter $c$ lower bound. For presentation we can drop the $\log{0.08104}$ (yielding a looser bound):
$$c \geq \frac{(1-\gamma) + R \epsilon_{\mathrm{ls}}\exp{\epsilon_P}}{(1 - \gamma^{T-1} (\epsilon_P + 2\log{\epsilon_r})) (1-\gamma)}$$
%Multiplying the guarantee of Proposition \ref{prop:ldm_degradation} by $\gamma^t$, we have that:
%$$\gamma^t \log{P}(s_t, a_t) \geq \log{c} - \frac{R \epsilon_{\mathrm{ls}}\exp{\epsilon_p}}{c(1-\gamma)} - \gamma^t \epsilon_p$$
%Since $P(s_t, a_t) \geq 0$, we have that $ \log{P}(s_t, a_t) \leq P(s_t, a_t) - 1$ implying that
%$$\gamma^t P(s_t, a_t) \geq (1-\epsilon_p) \gamma^t + \log{c} - \frac{R \epsilon_{\mathrm{ls}}\exp{\epsilon_p}}{c(1-\gamma)}$$
%Since $c > 0$ we also have that $\log{c} \geq 1- \frac{1}{c}$ so 
%$$\gamma^t P(s_t, a_t) \geq 1 + (1-\epsilon_p)\gamma^t - \frac{1 + (1-\gamma)^{-1}R \epsilon_{\mathrm{ls}}\exp{\epsilon_p} }{c}$$
%Summing until $T$, we have that
%$$\sum_{t=0}^T\gamma^t P(s_t, a_t) \geq (1-\epsilon_p) \cdot \frac{1 - \gamma^{T+1}}{1-\gamma} - T\cdot \left(\frac{1 + (1-\gamma)^{-1}R \epsilon_{\mathrm{ls}}\exp{\epsilon_p} }{c} - 1\right)$$
%denoting $\epsilon_\mathrm{dec} \doteq \frac{1 + (1-\gamma)^{-1} R \epsilon_{\mathrm{ls}} \exp{\epsilon_p}}{c} - 1$ yields the first stated result.
\end{proof}

The proof of Corollary~\ref{cor:ex_7.1_w_ldm_error} is a similarly fast consequence of Proposition~\ref{ex:baseline_failure}:
\begin{proof}[Proof of Corollary~\ref{cor:ex_7.1_w_ldm_error}]It is easy to check that the construction $D_{H, \epsilon}$ from Proposition~\ref{ex:baseline_failure} satisfies $\sup_{P^{(H+2)}}\|E-\hat{G}^\star\|_{P^{(H+2)}} \leq 2 \epsilon_p$ so by Proposition~\ref{generalized_degradation_bd}, with probability at least $1-2\delta_p$, 
$$\log P(s_t, a_t) \geq \log{\hat{c}} - \frac{(R\cdot (H+2) \cdot \epsilon_{\mathrm{ls}} + 2 \epsilon_p)\exp{\epsilon_p}}{\hat{c}}$$
for all $t$. Hence as long as $\log{\frac{\hat{c}}{\epsilon}} > \frac{(R\cdot (H+2) \cdot \epsilon_{\mathrm{ls}} + 2 \epsilon_p)\exp{\epsilon_p}}{\hat{c}}$ we are guaranteed to not exhibit the naive myopic behavior and instead ensure $P(s_t, a_t) \geq \frac{1}{2(H+1)}$ which is the best we could hope for. Since $\log x \geq 1 - x^{-1}$, it is enough to have $\hat{c}$ for which
$$\hat{c} \geq (R(H+2) \epsilon_{\mathrm{ls}} + 2\epsilon_p) \exp{\epsilon_p} + \epsilon$$
as desired.\end{proof}

We can also make the following remark about generalization to probability density functions.
\begin{remark}[Extension to Probability Density Functions]\label{rmk:pdf_ext}For $P$ being a p.d.f. (probability density function), we can extend the above simple approach via Markov's inequality which guarantees that $$P(|\hat{G}^\star(s,a) - \hat{G}(s,a)| \geq a) \leq \frac{R\epsilon_{\mathrm{ls}}}{(1-\gamma)a} +  \frac{\gamma^K \|E - \hat{G}^\star\|_{\infty}}{a}$$
This implies that the $\mathbb{R}^{d_s\times d_a}$ hypervolume of the set of unruly state-action pairs $\mathcal{U}_{a, c} = \{(s,a) \in \mathcal{S}\times\mathcal{A} \text{ s.t. } |\hat{G}^\star(s,a) - \hat{G}(s,a)| \geq a \text{ and } P(s, a) \geq c\}$ cannot be larger than $\frac{R\epsilon_{\mathrm{ls}}}{(1-\gamma)a} +  \frac{\gamma^K \|E - \hat{G}^\star\|_{\infty}}{c a}$, i.e. $$\text{Vol}(\mathcal{U}_{a, c}) \leq \frac{R\epsilon_{\mathrm{ls}}}{(1-\gamma)a} +  \frac{\gamma^K \|E - \hat{G}^\star\|_{\infty}}{c a}$$ and hence bound the size of the region where the previous approach does not work (of course additionally one needs to choose $a$ to get the best guarantee).\end{remark}

Finally, let us attempt to derive an analogue to Prop.~\ref{prop:ldm_degradation} for the density threshold within the horizon $H$ to better understand how the LDM compares to a standard density model. As discussed in the main text, we need to impose some regularity assumptions first.

\begin{assumption}\label{assum:lipschitzness}
Assume that $E$ is $L_p$-Lipschitz, that the true dynamics $f$ are $L_f$-Lipschitz in the state and that he model has error $|f(s,a) - \hat{f}(s,a)| \leq \frac{\epsilon_f}{P(s,a)}$.
\end{assumption}

\begin{proposition}\label{density_degradation}If Assumption~\ref{assum:lipschitzness} holds, from any state $s_0$, taking $a_0$ according to (\ref{density_constraint}) guarantees that $\forall t$, there is a sequence $a_{1:t}$ for which
$$\log P(s_t, a_t) \geq \begin{cases}
\log{c} - \frac{L_P \epsilon_f (1+\ldots+ L_f^{t-1}) \exp{\epsilon_p}}{c} - \epsilon_p \text{ if } t\leq H,\\
-\infty \text{    otherwise.}\end{cases}$$
\end{proposition}  

\begin{proof}
Let $a'_{1:t}$ be the solution to the following problem to the $t\leq H$ horizon MPC with $E$ and $\hat{f}$ and $s_t$ be the sequence of states reached by playing $a'_{1:t}$ on the real model $f$. Under Assumption~\ref{assum:lipschitzness}, we have:

\begin{align*}
-\log{P(s_t, a_t')} &\leq E_\theta(s_t, a_t') + \epsilon_p \\
&\leq E_\theta(s_t', a_t') + L_p \|s_t' - s_t\| + \epsilon_p \\
&\leq -\log c + \epsilon_p + \dfrac{L_p\cdot \epsilon_f \cdot (1+ \ldots + L_f^{t-1})}{c \cdot e^{-\epsilon_p}}\\
&\leq -\log c + \epsilon_p + \dfrac{L_p\cdot \epsilon_f \cdot (1+ \ldots + L_f^{t-1}) \cdot \exp{\epsilon_p}}{c}
\end{align*}
where we used line follows by:
where the last part follows because
\begin{align*}
||s_t - s_t'|| &= ||f(s_{t-1}, a_{t-1}') - \hat{f}(s_{t-1}', a_{t-1}')||   \\
&\leq ||f(s_{t-1}, a_{t-1}') - f(s_{t-1}', a_{t-1}')|| + ||f(s_{t-1}', a_{t-1}') - \hat{f}(s_{t-1}', a_{t-1}')|| \\
&\leq L_f ||s_{t-1} - s_{t-1}'|| + \frac{\epsilon_f}{c e^{-\epsilon_p}} \\
&\leq L_f^2 ||s_{t-2} - s_{t-2}'|| + \frac{\epsilon_f}{c e^{-\epsilon_p}} (1 + L_f) \\
&\leq ... \\
&\leq 0 + \frac{\epsilon_f (1+\ldots + L_f^{t-1})}{c e^{-\epsilon_p}}
\end{align*}
\end{proof}

Proposition~\ref{density_degradation} shows a similar behavior to the LDM guarantee in the short-horizon. As suggested by \citet{SzepesvariFPI}, the constants in analyses on fitted Q iteration are interwined with Lipschitzness of the system and of the probability function. In our case this relation springs forth more naturally since we are considering a dynamical system rather than an MDP.

\section{Appendix to Section~\ref{sec:experiments}}
\label{appendix:sec_9}
\subsection{Data and Tasks}
\label{appendix:data}
Our experiments consist of a goal-directed hopper task, a goal-directed lunar lander task, and SimGlucose \cite{simglucose} , a blood glucose setpoint task. %, and \textcolor{red}{[ventillator task]}.
Here we present details about the task setup and training data for each domain. 

\subsubsection{Hopper}
The goal-directed hopper task involves controlling an agent with 12 state dimensions and 3 action dimensions. The agent is randomly initialized in a x position in $[-3, 3]$, and the goal of the task is to control the agent to go to a target x location in $\{-2, -1, 0, 1, 2\}$. In our experiments, we tested each method 10 times (each from a randomly initialized starting position) for each goal location. The specific reward function we used for planning is $r(s, a) = -|s_x - s_g|$, where $s_x$ is the x position of the agent, and $s_g$ is the goal position. 

The dataset we used to train our agents consists of stochastic expert behavior, where the experts were trained to go towards a random goal location within $[-4, 4]$ of the agent's initial position. The expert behaviors included in the dataset include hopping forwards, backwards, and stopping. In order to increase the coverage of the dataset, we added random noise to the expert's actions at each timestep. More specifically, for each trajectory, we randomly sample $Y \sim \text{Unif}(-1.5, 1.5)$, and at each timestep we add $Z \sim \text{Normal}(0, |Y|)$ to the expert's chosen actions. The size of the dataset is 8499511 transitions. 

\subsubsection{Lunar Lander}
\label{appendix:data_ll}
The goal-directed lunar lander task involves controlling an agent with 8 state dimensions and 2 action dimensions. The agent is always initialized in the middle-top of the screen, and the goal of the task is to control the aircraft to land on a target landing pad without crashing. The location of the landing pad is centered at either 4 or 6 on the horizontal axis. In our experiments, we tested each method 10 times for each landing pad location. The specific reward we used for planning is $r(s, a) = -\sqrt{(s_x - s_g)^2+s_y^2}$, where $s_x$ is the horizontal position of the agent, and $s_g$ is the landing pad position, and $s_y$ is the height of the agent.  

The dataset we used to train our agent consists of stochastic expert behavior, where the expert was trained to land in a random location between 1 and 10. Thus, the expert behavior in the dataset include going towards and landing in different landing pad locations. In order to increase the coverage of the dataset, we added random noise sampled from $Z \sim \text{Normal}(0, 0.9)$ to the expert's actions at each timestep. The size of the dataset is 1650022 transitions. For a visualization of the task, see Fig \ref{lunar_lander_visualization}.

\begin{figure}
\centering
\includegraphics[width=.4\columnwidth, trim={0 0 0 0},clip]{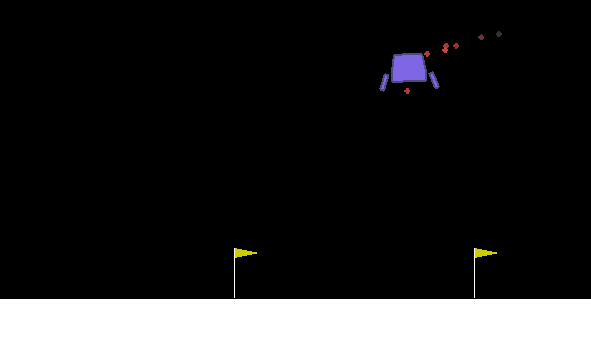}
\caption{Visualization of the lunar lander task.}
\label{lunar_lander_visualization}
\end{figure}

\subsubsection{SimGlucose}

The simglucose task dataset consists of a patient having their blood glucose level controller by a reference PID controller (P=0.001, I=0.00001, D=0.001) set with target $140$. The state dimension is 2 (representing current blood glucose and current caloric input) and the action dimension is 1 (representing insulin administered). In our experiments, we tested each method 100 times. In order to be consistent with the PID approach as well as with our other experiment, we use the reward $r(\mathrm{bg}) = -|\mathrm{bg} - 140|$ during planning. This reward serves to enforce staying close to the midpoint of the safety region $(100, 180)$, matching the reference PID's target setpoint, as well as allowing for a formulation that ties setpoint maintanence and goal reaching tasks together. Note that the controller must be able to neither under nor over react to being suddenly in a more dangerous range. This is a particularly well-suited task for the LDM approach since the goal is to leverage past data and learn to filter what actions to take to ensure not getting into a hypo/hyperglycemia alert, while also trying to get to the target blood glucose level as fast as possible. The size of the dataset is $1\mathrm{mil}$ transitions. 

%\subsubsection{\textcolor{red}{ventillator}}\textcolor{red}{[FILL]}

% The state dimension, action dimension, and number of transitions in both datasets are presented below. 

\newcommand{\specialcell}[2][c]{%
  \begin{tabular}[#1]{@{}c@{}}#2\end{tabular}}

% \begin{center}
% \begin{tabular}{|c|c|c|c|c|}
%  \hline
%  Task & \specialcell{State \\ dimension} & \specialcell{Action \\ dimension} & \specialcell{Num \\ transitions} \\ 
%   \hline
%  Hopper & 12 & 3 & 8954033 \\ 
%  \hline
%  Glucose  & 1 & 1 & 135500000 \\ 
%  \hline
% \end{tabular}
% \end{center}

\subsection{Flow Model Training}
For all density models, we made use of a neural spline flow \cite{durkan2019neural}. The training hyperparameters are presented below.

\begin{center}
\begin{tabular}{|c|c|c|c|c|c|c|c|c|c|}
 \hline
 task &  \specialcell{flow \\ layers} &  \specialcell{knots \\ K}& \specialcell{bounds \\ B} & coupling transform & MLP & \specialcell{learning\\ rate} & \specialcell{weight\\ decay} & optimizer & \specialcell{train\\ steps} \\ 
  \hline
 Hopper & 4 & 64 & 20&rational-quadratic & 256-256-256 &1e-4 &1e-5 & Adam & 150k \\ 
 \hline
 \specialcell{Lunar\\ Lander} & 4 & 64 & 20&rational-quadratic & 256-256-256 &1e-4 &1e-5 & Adam & 150k \\ 
 \hline
  Glucose & 2 & 64 & 20&rational-quadratic & 256-256-256 &1e-4 &1e-5 & Adam & 500k\\
    \hline
%\textcolor{red}{ Ventillator} & 2 & 64 & 20&rational-quadratic & 256-256-256 &1e-4 &1e-5 & Adam \\ \hline
\end{tabular}
\end{center}

We used the output of the flow models to label our datasets for LDM training. In order to make training more robust, we manually labeled the density of early termination points to have low density for the Hopper and Lunar Lander experiments. More specifically, we labeled them to be $\min(\{E_\theta(s_i, a_i)\}_{i=1}^N)-3*\text{std}(\{E_\theta(s_i, a_i)\}_{i=1}^N)$, where $\{E_\theta(s_i, a_i)\}_{i=1}^N$ denotes the dataset's labels from the flow model. 

\subsection{LDM Training}
\label{appendix:ldm_training}
In this section, we present the hyperparameters we used in our experiments to train an LDM, presented in Alg. \ref{ldm_alg}.

First, present the training parameters for the LDM function $G_\phi(s, a)$ and its associated policy $\pi_\psi(s)$. $G_\phi(s, a)$ outputs a scalar value, and $\pi_\psi(s)$ outputs a distribution over actions, parameterized by the mean and variance of a gaussian. 

\begin{center}
\begin{tabular}{|c|c|c|c|c|}
 \hline
 \multicolumn{5}{|c|}{$G_\phi(s, a)$ training parameters}\\
 \hline
 task & batch size & architecture &learning rate & optimizer \\ 
  \hline
 Hopper & 256 & 256-256 & 3e-4&Adam\\ 
   \hline
 Lunar Lander & 256 & 256-256 & 3e-4&Adam\\ 
 \hline
 Glucose & 256 & 256-256 & 3e-4&Adam\\ 
 \hline
 %\textcolor{red}{Ventilator} & 256 & 256-256 & 3e-4&Adam\\ \hline
\end{tabular}
\end{center}

\begin{center}
\begin{tabular}{|c|c|c|c|c|}
 \hline
 \multicolumn{5}{|c|}{$\pi_\psi(s)$ training parameters}\\
 \hline
 task & batch size & architecture &learning rate & optimizer \\ 
  \hline
 Hopper & 256 & 256-256 & 1e-4&Adam\\ 
 \hline
  Lunar Lander & 256 & 256-256 & 1e-4&Adam\\ 
 \hline
 Glucose & 256 & 256-256 & 1e-4 &Adam\\ 
 \hline
 %\textcolor{red}{Ventilator}  & 256 & 256-256 & 1e-4&Adam\\  \hline
\end{tabular}
\end{center}

Next, we discuss the details of the conservative regularization term CQL($\mathcal{H}$) \cite{kumar2020conservative} we use to enable more stable offline training. Our implementation of the regularizer involves importance sampling from $\pi_\psi(\st)$, $\pi_\psi(\stp)$, and a uniform distribution over the action space, more specifically:

\begin{align}
    &\text{CQL}(\mathcal{H}) \\
     = &-\beta \log \sum_{a \in \mathcal{A}} \exp G(s, a) \\
    \approx &- \beta \left( \frac{1}{3N_{CQL}} \sum_{a_i \sim \text{Unif}(a)}^{N_{CQL}}\left[\frac{\exp G(s, a_i)}{\text{Unif}(a)}\right]
    +\frac{1}{3N_{CQL}} \sum_{a_i \sim \pi_\psi(.|\st)}^{N_{CQL}}\left[\frac{\exp G(s, a_i)}{\pi_\psi(.|\st)}\right] \right. \\
    & \left. +\frac{1}{3N_{CQL}} \sum_{a_i \sim \pi_\psi(.|\stp)}^{N_{CQL}}\left[\frac{\exp G(s, a_i)}{\pi_\psi(.|\stp)}\right]\right)
\end{align}
Note that this is negative of the regularizer presented in the original work on CQL, because the authors there were concerned with regulating Q functions in the reinforcement learning setting, where higher values correspond to higher rewards. In our setting, however, lower values in the LDM correspond to higher densities (lower negative log probabilities), so this regularizer is inverted in our setting. The specific hyperparameters we use for this regularizer are:

\begin{center}
\begin{tabular}{|c|c|c|}
 \hline
 task & $\beta$ & $N_{CQL}$  \\ 
  \hline
 Hopper & 10 & 10 \\ 
 \hline
  Lunar Lander & 1 & 10 \\ 
 \hline
 Glucose & 1 & 10\\ 
 \hline
 % \textcolor{red}{Ventilator} & 5 & 10 \\ \hline
\end{tabular}
\end{center}

Finally, we present the algorithm hyperparameters we used:

\begin{center}
\begin{tabular}{|c|c|c|c|c|c|c|}
 \hline
 task & $\gamma$ & $\tau$ & target entropy & $\alpha$ learning rate & $\alpha$ optimizer & train steps  \\ 
  \hline
 Hopper & 1 & 5e-3 & 20 & 1e-4 & Adam & 3000000 \\ 
   \hline
 Lunar Lander & 1 & 5e-3 & 2 & 1e-4 & Adam & 500000 \\ 
 \hline
 Glucose & 0.9 & 5e-3 & 0 & 1e-4 & Adam & 200000\\ 
 \hline
 % \textcolor{red}{Ventilator} & 0.9 & 5e-3 & 0 & 1e-4 & Adam & 20000 \\ \hline

\end{tabular}
\end{center}

\subsection{Dynamics Model Training}
In this section, we present the hyperparameters we used to train the dynamics model:

\begin{center}
\begin{tabular}{|c|c|c|c|c|c|c|}
 \hline
 task & architecture & batch size & learning rate & weight decay & optimizer & train steps  \\ 
  \hline
 Hopper & 256-256 & 256 & 3e-4 & 1e-5 & Adam & 50000 \\ 
   \hline
 Lunar Lander & 256-256 & 256 & 3e-4 & 1e-5 & Adam & 50000 \\ 
 \hline
 Glucose & 256-256 &256 & 3e-4 & 1e-5 & Adam & 200000\\ 
 \hline
 % \textcolor{red}{Ventilator} & 256-256 &256 & 3e-4 & 1e-5 & Adam & 100000\\ \hline
\end{tabular}
\end{center}

Additionally, in order to effectively learn a dynamics model for hopper, the highest dimensional system of the 4 domains, we preprocessed the training data (subtracting mean and dividing by standard deviation), and trained the neural network to output the change in state rather than the next state itself (predicting $s_{t+1} -s_t$ where $s_t$ is the current state). For all other tasks, we directly trained on the (unnormalized) training data and predicted the next state, because this procedure was sufficient for learning a good dynamics model within the data distribution.  

\subsection{Planning}
\label{app:planning}

In the model-based reinforcement learning algorithm presented in Section \ref{sec:ldm_control}, we use the learned dynamics model to perform planning. In this section, we present details about the planning procedure used in our experiments. 

The action trajectory that maximizes the optimization problem in \eqref{mbrl_opt_prob} can be approximated with stochastic zeroth-order optimization. More specifically, we generate random action trajectories uniformly sampled from a sampling prior, and select the trajectory which maximizes the objective while satisfying a constraint that is dependent on the method. More details about the sampling prior, optimization objective and constraint are described below. If no randomly generated action trajectory satisfies the constraint, we execute the action from a learned policy which was trained to optimize the constraint function.

\subsubsection{MPC parameters}
In this section, we present the parameters we used to perform planning via MPC. 
\begin{center}
\begin{tabular}{|c|c|c|c|c|}
 \hline
 task & sampling prior & objective & num random actions & horizon   \\ 
  \hline
 Hopper & \specialcell{normal distribution fitted around \\action marginal of the data} & $-|s_x - s_g|$ & 1024 & 1 \\ 
  \hline
 Lunar Lander & uniform distribution over action space & $ -\sqrt{(s_x - s_g)^2+s_y^2}$ & 8000 & 1 \\ 
 \hline
  Glucose & uniform distribution over action space & $-|\mathrm{bg} - \mathrm{bg}^\star|$ & 1024 & 1\\ 
 \hline
  % \textcolor{red}{Ventilator} & uniform distribution over action space & $-|\mathrm{bg} - \mathrm{bg}^\star|$ & 1024 & 1\\ \hline
\end{tabular}
\end{center}

\subsubsection{Constraints used in MPC}
In addition to our method, with performs MPC with an LDM constraint, we evaluate our experiment on 2 baseline methods, one with a density model constraint, and one with a ensemble constraint. For the ensemble constraint, we train $n_{\text{ensemble}}$ independent dynamics models $\{f_\xi(s, a)_i\}_{i=0}^{n_{\text{ensemble}}}$ by minimizing the loss \eqref{eq:model_erm_loss}, and use the negative variance of the outputs of the ensemble models as the constraint function. In our experiments, we use $n_{\text{ensemble}}=5$. We summarize the exact equations we use to constrain MPC for each type of constraint in the table below.
\begin{center}
\begin{tabular}{|c|c|}
 \hline
 constraint type & constraint function \\ 
  \hline
 LDM & $Q_\phi(s, a)$  \\ 
 \hline
  density model &  $E_\theta(s, a)$\\ 
 \hline
 ensemble  &  $-\text{variance}(\{f_\xi(s, a)_i\}_{i=0}^{n_{\text{ensemble}}})$\\ 
 \hline
\end{tabular}
\end{center}

\subsubsection{Backup policy}
When performing constrained optimization with zeroth order optimization, it is possible for none of the sampled actions to pass the constraint. In this case, we execute the action outputted by a learned policy which is optimized to maximize the constraint used for each method. 

For the LDM constraint, the backup policy is $\pi_\psi(\st)$, the policy associated with the LDM, as described in Sec. \ref{sec:learning_ldms}.

For the backup policies for the density model and ensembles constraint, we train the backup policy in a similar fashion, by optimizing to output the action which will maximize its value under the constraint. More specifically, the loss function is $L_{\pi}(\psi) = \mathbb{E}_{\st \sim p_D}[W_{\phi}(\st, \pi_{\psi}(\st))] - \alpha \mathcal{H}(\pi_\psi(.|\st))$, where $W(s, a)$ represents the constraint function being used. The Algorithm for training and the hyperparameters we used for training are presented below. The main difference between this procedure and Alg. \ref{ldm_alg} is that the constraint functions here are held \emph{fixed}, whereas in Alg. \ref{ldm_alg}, the LDM function and policy were updated concurrently. 

\begin{algorithm}[H]
\SetAlgoLined
 initialize parameter vectors  $\psi$, and $\alpha$\\
 \For{\text{num backup policy training steps}}{
  $\psi \gets \psi - \lambda_\pi \nabla_{\psi} L_\pi(\psi) $\\
  $\alpha \gets \alpha - \lambda_\alpha \nabla_{\alpha} L(\alpha)$

 }
 \caption{Backup Policy Training}
\end{algorithm}

\begin{center}
\begin{tabular}{|c|c|c|c|c|c|c|c|}
 \hline
 task &batch size & architecture & $\psi$ learning rate & target entropy & $\alpha$ learning rate & optimizer & train steps  \\ 
  \hline
 Hopper & 256 & 256-256 & 3e-4 & 3 & 3e-4 & Adam & 500000 \\ 
 \hline
  Lunar Lander & 256 & 256-256 & 3e-4 & 2 & 3e-4 & Adam & 500000 \\ 
 \hline
 Glucose & 256 & 256-256 & 3e-4 & 1 & 3e-4 & Adam & 250000 \\ 
 \hline
 % \textcolor{red}{Ventilator} & 256 & 256-256 & 3e-4 & 1 & 3e-4 & Adam & 250000 \\ \hline
\end{tabular}
\end{center}

\subsection{Ablation Experiments on the Effect of Discount Factor $\gamma$}
\label{app:ablation_discount_factor}
We performed ablation experiments to investigate the effect of using a discount factor on LDM training for the lunar lander task, described in Section \ref{appendix:data_ll}. We tried discount factors of 1, 0.99, 0.95, and 0.9, and compare against using a density model constraint. In Fig. \ref{discount_factor_plots}, we show the average reward (left), and failure rate (right) of each method.

\begin{figure}
\centering
\includegraphics[width=.4\columnwidth, trim={0 0 0 40},clip]{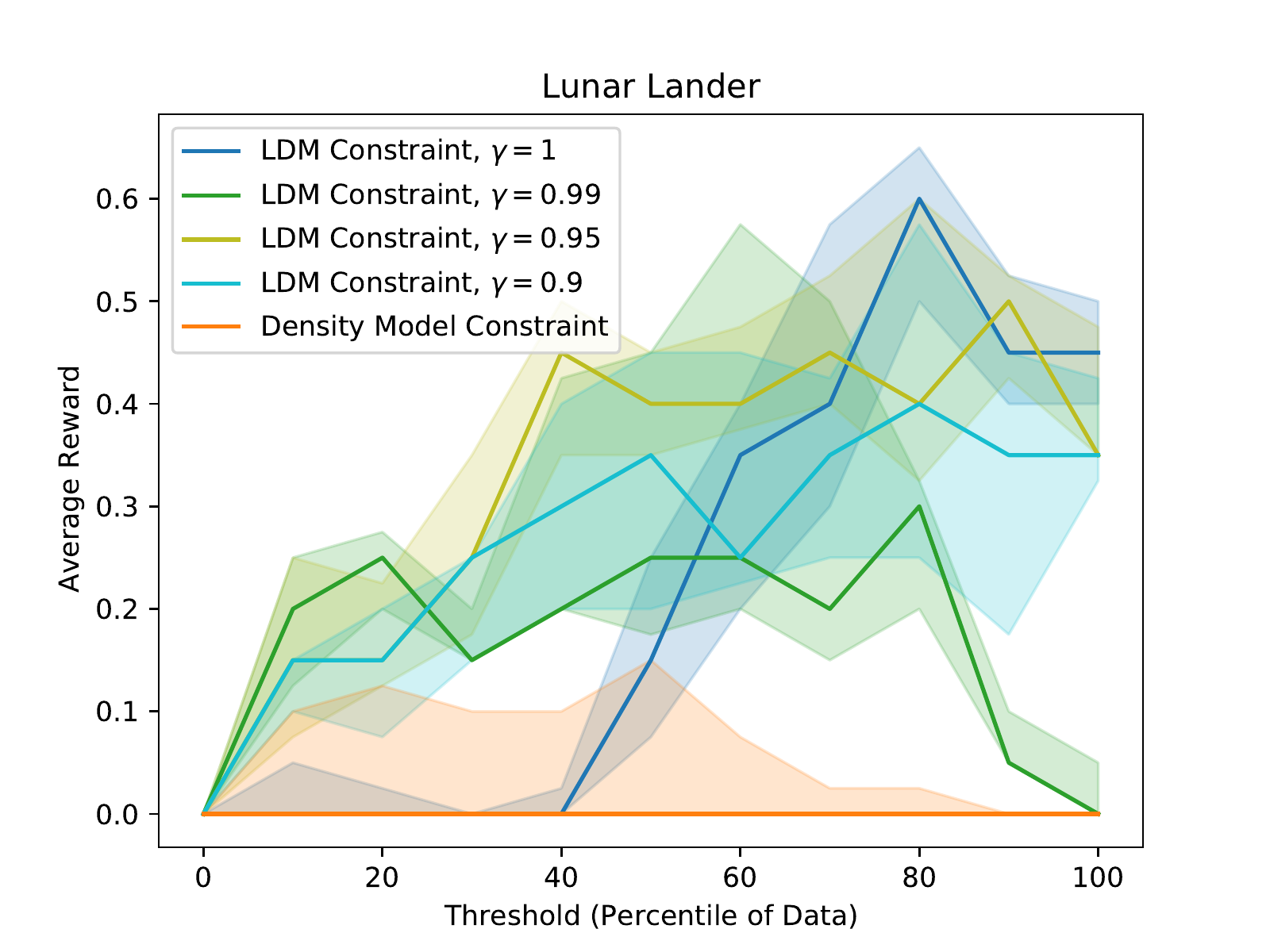}
\includegraphics[width=.4\columnwidth, trim={0 0 0 0},clip]{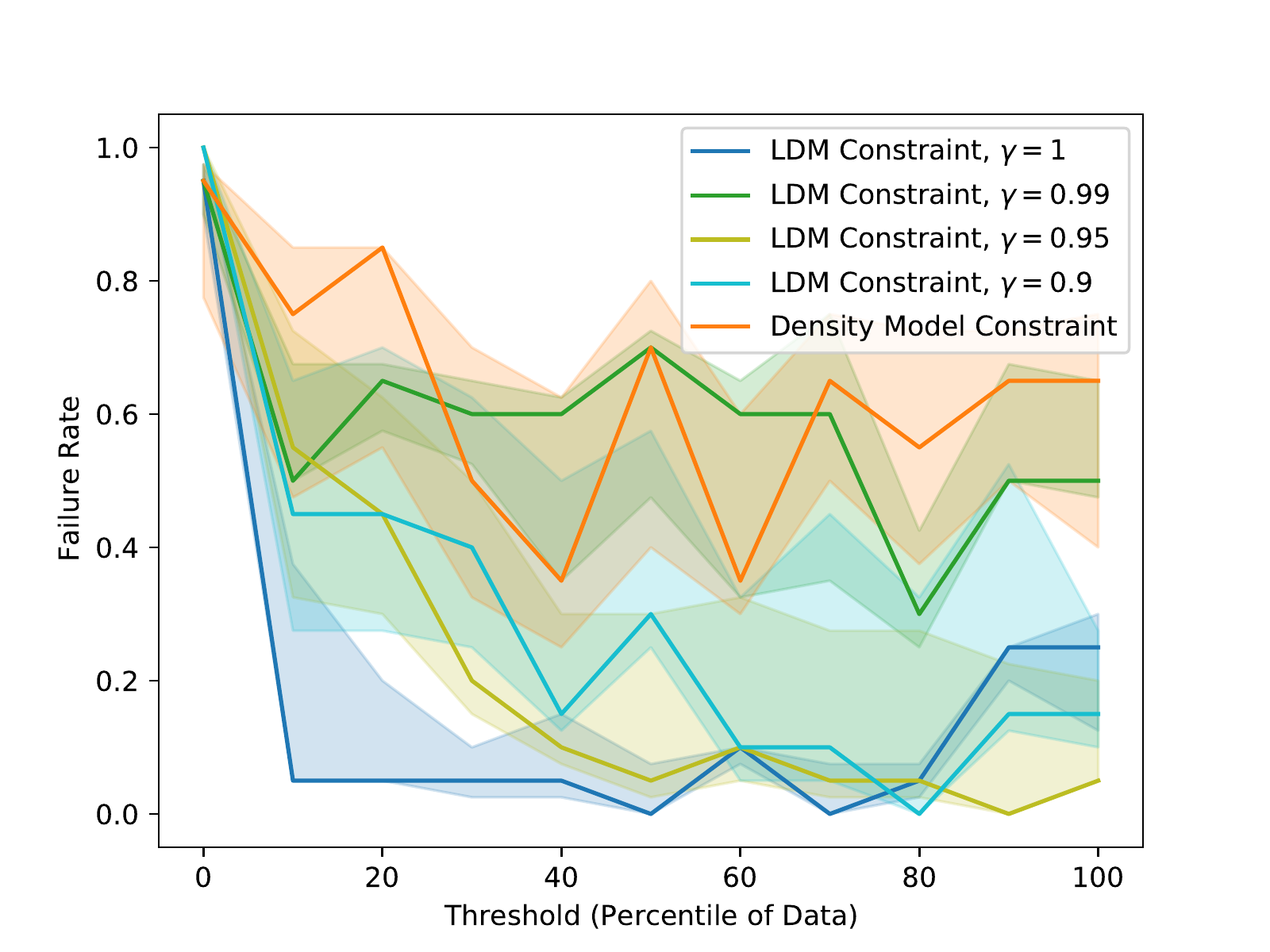}
\caption{Effect of discount factor on LDM performance for the lunar lander task. On the left, we show the average reward for each method, and on the right, we show the failure rate for each method. Each method was evaluated with 3 random seeds.}
\label{discount_factor_plots}
\vspace{-1em}
\end{figure}

\subsection{Full Experimental Results Across Threshold Sweep}
\label{appendix:full_sweep}
Due to space constraints, we were unable to include the full sweep of threshold values vs. average reward and failure rate for all the domains in the main paper. In Fig. \ref{full_plots}, we include the full sweeps for the lunar lander, SimGlucose, and ventilator tasks for a more comprehensive presentation of the performance of the different methods. 

\begin{figure}
\centering
\includegraphics[width=.4\columnwidth, trim={0 0 0 0},clip]{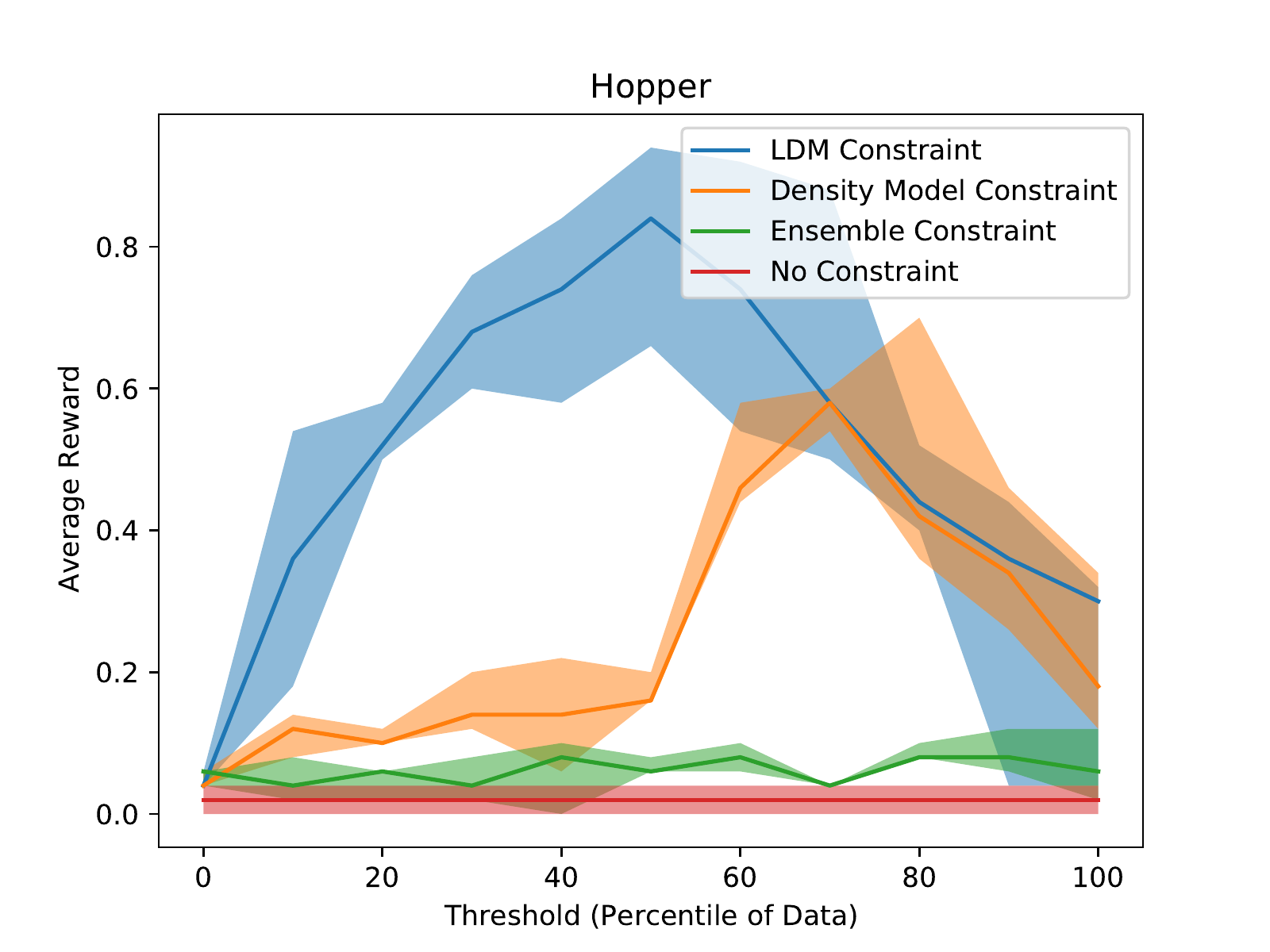}
\includegraphics[width=.4\columnwidth, trim={0 0 0 0},clip]{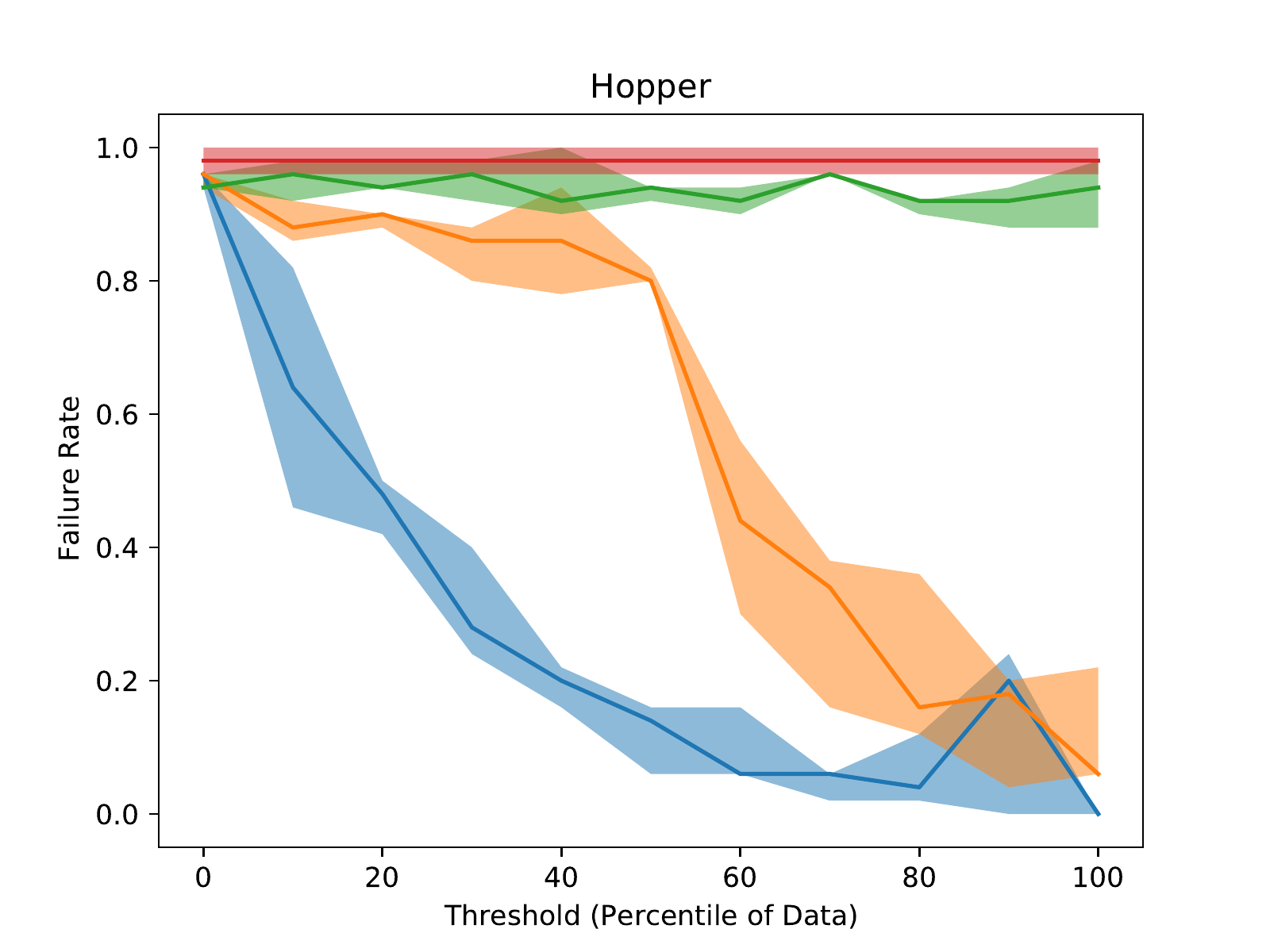} \\
\includegraphics[width=.4\columnwidth, trim={0 0 0 0},clip]{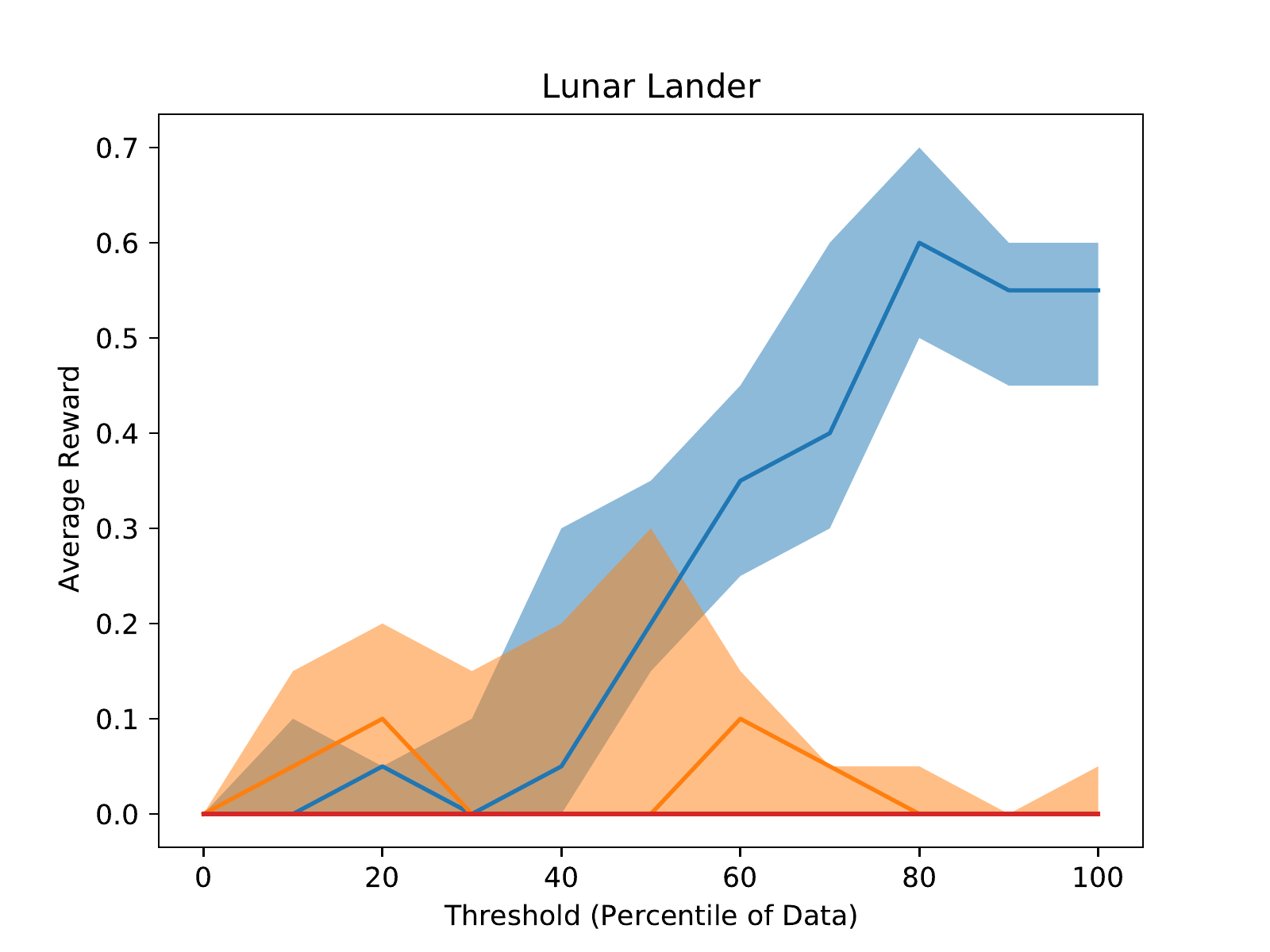}
\includegraphics[width=.4\columnwidth, trim={0 0 0 0},clip]{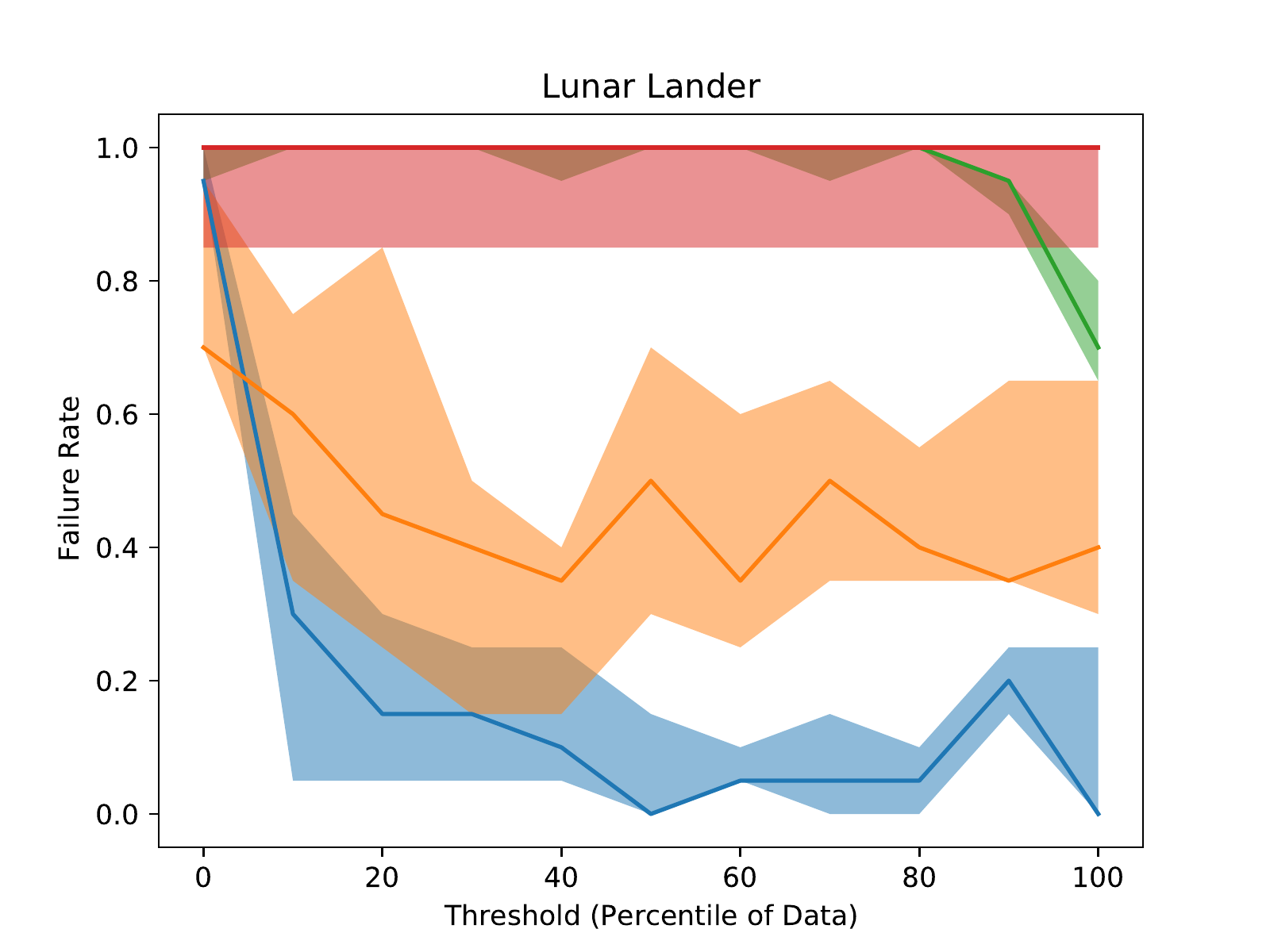} \\
\includegraphics[width=.4\columnwidth, trim={0 0 0 0},clip]{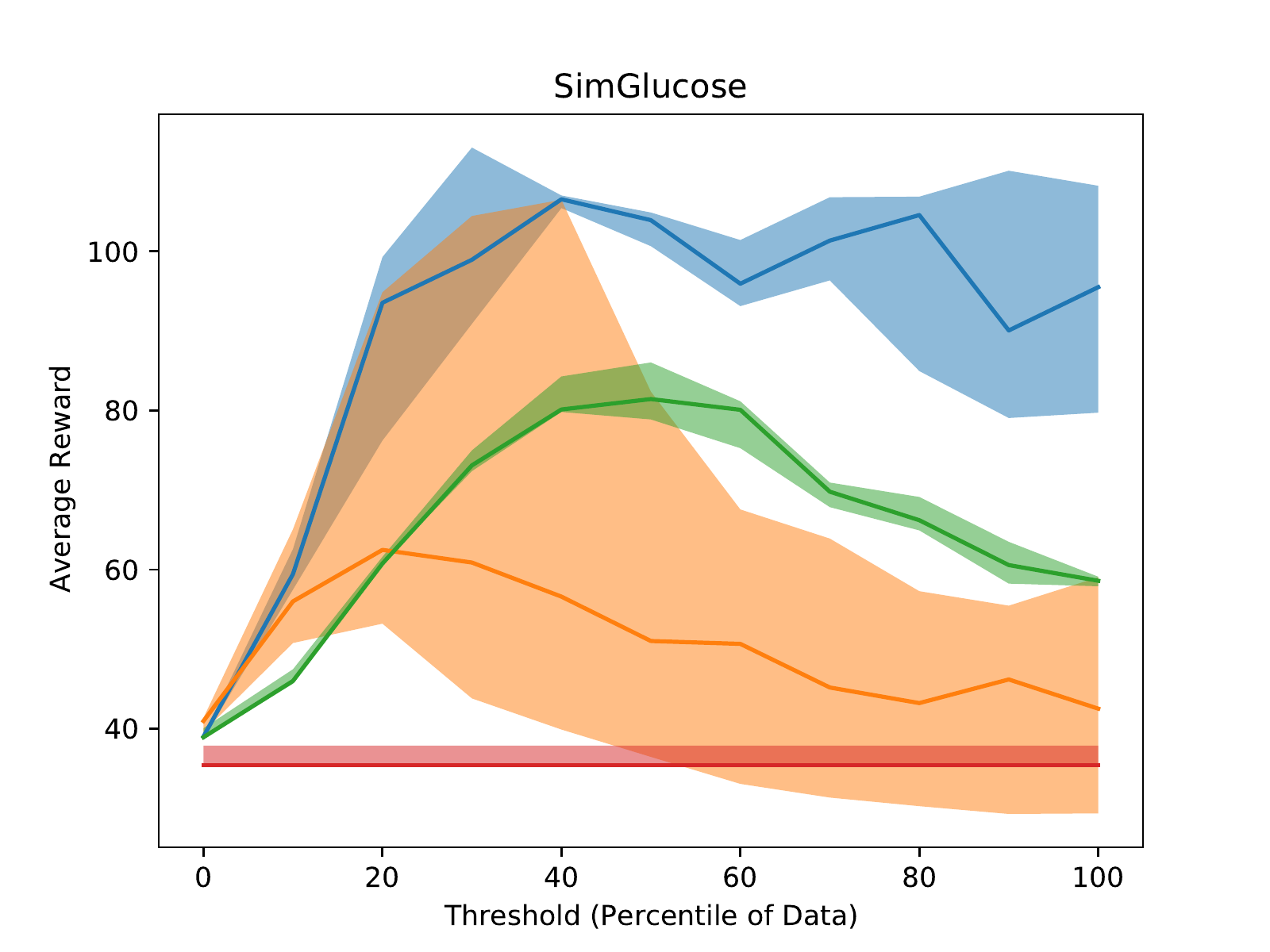}
\includegraphics[width=.4\columnwidth, trim={0 0 0 0},clip]{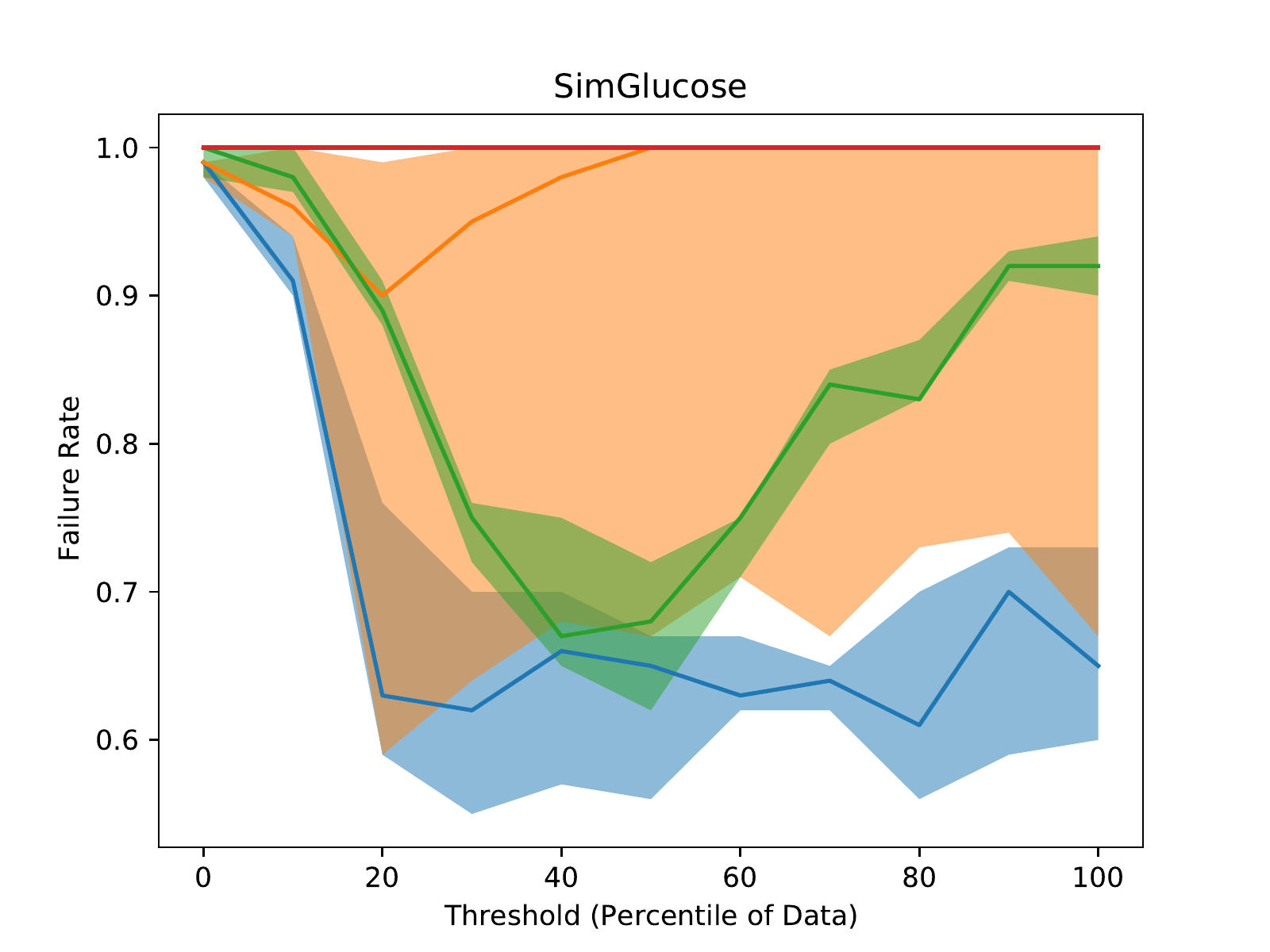} \\
\caption{Evaluation of average reward and failure rate over different threshold values for the hopper, lunar lander, and SimGlucose tasks. The x axis of the plots on the left represents the percent of the dataset whose value under the constraint function falls below the associated threshold $c$. More specifically, $x = \text{percentile}(\{W(\st^i, \at^i)\}_{i=1}^N, c)$, where $W$ is the constraint function. We use this representation of constraint values in order to make them comparable across different constraint functions. The solid line represents the median performance over random seeds, and the top and bottom of the semi-transparent region indicates the 25/75th percentiles. We ran 5 random seeds for each tasks. }
\label{full_plots}
\vspace{-1em}
\end{figure}

\section{Guarantees under Partial Observability}\label{app:partial}

The LDM as formulated above addresses fully observed tasks with Markovian state spaces. In many cases, the state observations may not be perfectly Markovian, and the system may be partially observed. We can extend the LDM to non-Markovian state observations, and retain appealing guarantees. Let $\mathcal{O}$ be the observation space and $h:\mathcal{S} \rightarrow \mathcal{O}$ the function that maps a state to an observation. For the remainder of the partially observed setting, we will slightly overload the notation $P$ to refer to the density on $\mathcal{O} \times \mathcal{A}$. We analyze the behavior of the same algorithm as before when deployed under partial observability. 

Intuitively, in the partially observed case, the LDMs can be seen as operating on a dynamical system that captures the \emph{expected} transitions in observation space. Formally, we define $f^\mathrm{obs}:\mathcal{O}\times\mathcal{A}\rightarrow \mathcal{O}$ as $f^\mathrm{obs}(o, a) = \mathbb{E}_{s|o} [h(f(s,a))]$, which we can see as a \textit{reduced} system capturing the expected trajectory progression, where the expectation is over all possible underlying states. More concretely, we are taking the entire distribution over possible states (which will only get narrower if we include history) and use the expected value under that distribution to define an associated dynamical system. Note that the maximal LDM for the reduced system constructed in this manner can be obtained with essentially the same algorithm as before. In a sense, we are able to maintain the exact same approach, and focus only what this existing algorithm guarantees under this setting. Formally, we will update the LDM using the empirical estimate of the full information backup operator $
\mathcal{T} G(o, a) = \max\{E(o,a), \min_{a'\in\mathcal{A}}\gamma G(f^\mathrm{obs}(o, a), a')\}$. The only difference is that, by sampling $(o, a, o')$, we may have $o'\neq f^\mathrm{obs}(o, a)$. As a result we cannot directly fit the above and instead go over each transition in the dataset and fit the following partial information back-up estimate $\widehat{\mathbb{E}\mathcal{T}}$,
the \textit{sampled} counterpart to Eq. \ref{eq:bellman_op_density} (averaging over possible next observations):
\begin{equation*}\label{eq:ldm_partial_backup_estimate}
\widehat{\mathbb{E}\mathcal{T}} G(o, a)= \frac{\sum_{\mathcal{O}_D(o,a)} \max\{E(o,a), \min_{a'\in\mathcal{A}}\gamma G(o', a')\} }{|\mathcal{O}_D(o,a)|} 
\end{equation*}
where $\mathcal{O}_D(o,a) = \{o'\in \mathcal{O} \text{ s.t. } (o, a, o') \in D\}$. Plugging this into update (\ref{ldm_update}) yields:
\begin{equation*}
\hat{G}_{k+1} \in \arg\min\limits_{G\in\mathcal{G}} \sum_{(o, a) \in D} \left(G(o, a) -\widehat{\mathbb{E}\mathcal{T}} G_k(o, a)\right)^2    
\end{equation*}
Given that the LDM has a natural extension to this setting, the main insight of this section is that the same procedure as before still provides a distributional shift guarantee, but with an additional error term that accounts for the variability in the observation for the same underlying state. This is captured in the following proposition:
\begin{proposition}\label{prop:partial} Let $\hat{G}_0 = E$. The result of applying the sample bellman update $K$ times satisfies:
\begin{equation*}
\|\hat{G}_K - \hat{G}^\star\|_P \leq \dfrac{R \cdot \epsilon_{\mathrm{ls, obs}}}{1-\gamma} + \gamma^K \cdot \|E - \hat{G}^\star\|_{\infty},
\end{equation*}
for $\epsilon_{\mathrm{ls,obs}} \doteq \max\limits_{t\in[K-1]}\|\hat{G}_{t+1} - \widehat{\mathbb{E}\mathcal{T}}\hat{G}_t\|_P + \|\mathcal{T}\hat{G}_t - \widehat{\mathbb{E}\mathcal{T}}\hat{G}_t\|_P$.
\end{proposition}

Observe that the difference from the guarantee in Prop.~\ref{prop:ldm_fqi} is the additional $\|\mathcal{T}\hat{G}_t - \widehat{\mathbb{E}\mathcal{T}}\hat{G}_t\|_P$ term which captures the errors due to fitting $\widehat{\mathbb{E}\mathcal{T}}$ rather than $\mathcal{T}$, since the latter cannot be computed due to lack of access to $f^{\mathrm{obs}}$. As a result, we are able to formalize and motivate the behavior of our algorithm on partially observed systems, showing that: (i) the LDM is useful for bounding the probability of staying in-distribution under the expected transitions, and (ii) the convergence rate depends on how well the sampled Bellman operator estimate matches the true Bellman operator under the $f^{\mathrm{obs}}$. Note that if there is no information loss, then the bounds match the previous derivations exactly.

Further, if we assume there is some some constant $M > 0$ such that the magintude of the second derivative of $\mathcal{T} \cdot G$ for any $G\in\mathcal{G}$ is bounded by $M$,
we can guarantee that $\epsilon_{\mathrm{ls, obs}} =\mathcal{O}\left(\sqrt{\frac{M^2 \sigma_{\mathrm{obs}}}{|D|}}\right)$, where $\sigma_{\mathrm{obs}} = \sup_{o, a}\text{Var}_{s|o, a}[h(f(s,a))]$ captures the maximum variance of observation transitions. Note that $\sigma_{\mathrm{obs}}\rightarrow 0$ is equivalent to the observation capturing more and more of the true state, and hence results in the partially observed guarantee converging to its fully observed counterpart as we would expect.

%%%%%%%%%%%%%%%%%%%%%%%%%%%%%%%%%%%%%%%%%%%%%%%%%%%%%%%%%%%%%%%%%%%%%%%%%%%%%%%
%%%%%%%%%%%%%%%%%%%%%%%%%%%%%%%%%%%%%%%%%%%%%%%%%%%%%%%%%%%%%%%%%%%%%%%%%%%%%%%

\end{document}